\documentclass{article}

\usepackage{hyperref}

\usepackage{microtype}
\usepackage{graphicx}
\usepackage{booktabs}
\usepackage{amsmath}
\usepackage{amsfonts}
\usepackage{amssymb}
\usepackage{pifont}
\usepackage{amsthm}
\usepackage{subfig}
\usepackage{array,multirow,graphicx}
\usepackage{enumitem}
\usepackage{thm-restate}

\RequirePackage{algorithm}
\RequirePackage{algorithmic}

\setitemize{noitemsep,topsep=1.0pt,parsep=1.0pt,partopsep=1.0pt}
\setenumerate{noitemsep,topsep=1.0pt,parsep=1.0pt,partopsep=1.0pt}

\DeclareMathOperator*{\argmin}{arg\,min}
\DeclareMathOperator*{\argmax}{arg\,max}

\usepackage[dvipsnames]{xcolor}
\definecolor{ao}{rgb}{0.0, 0.0, 1.0}
\definecolor{ao(english)}{rgb}{0.0, 0.5, 0.0}

\definecolor{mydarkblue}{rgb}{0,0.08,0.45}
\hypersetup{ %
    colorlinks=true,
    linkcolor=mydarkblue,
    citecolor=mydarkblue,
    filecolor=mydarkblue,
    urlcolor=mydarkblue
}

\newcommand{\R}{\mathbb{R}}

\newcommand{\cX}{\mathcal{X}}
\newcommand{\cY}{\mathcal{Y}}
\newcommand{\cE}{\mathcal{E}}

\DeclareMathOperator*{\E}{\mathbb{E}}

\newcommand{\adult}{\texttt{Adult}}
\newcommand{\polaritydata}{\texttt{rt-polaritydata}}
\newcommand{\threshold}{\texttt{threshold}}

\newtheorem{theorem}{Theorem}
\newtheorem{claim}{Claim}

\newtheorem{definition}{Definition}
\newtheorem{corollary}{Corollary}

\title{Framework for Evaluating Faithfulness of Local Explanations}

\author{
    Sanjoy Dasgupta\\
    {\small University of California San Diego}\\
    {\small \texttt{dasgupta@eng.ucsd.edu}}
    \and
    Nave Frost\\
    {\small Tel Aviv University}\\
    {\small \texttt{\quad navefrost@mail.tau.ac.il}}
    \and
    Michal Moshkovitz\footnote{This work was partially done while the author was at the University of California San Diego.}\\
    {\small Tel Aviv University}\\
    {\small \texttt{moshkovitz5@mail.tau.ac.il}}
}
\date{}

\begin{document}

\maketitle

\begin{abstract}
We study the faithfulness of an explanation system to the underlying prediction model. We show that this can be captured by two properties, {\bf consistency} and {\bf sufficiency}, and introduce quantitative measures of the \emph{extent} to which these hold. Interestingly, these measures depend on the test-time data distribution.
For a variety of existing explanation systems, such as anchors, we analytically study these quantities. We also provide estimators and sample complexity bounds for empirically determining the faithfulness of black-box explanation systems. 
Finally, we experimentally validate the new properties and estimators. 
\end{abstract}

\section{Introduction}
Machine learning is an integral part of many human-facing computer systems and is increasingly a key component of decisions that have profound effects on people's lives. There are many dangers that come with this. For instance, statistical models can easily be error-prone in regions of the input space that are not well-reflected in training data but that end up arising in practice. Or they can be excessively complicated in ways that impact their generalization ability. Or they might implicitly make their decisions based on criteria that would not considered acceptable by society. For all these reasons, and many others, it is crucial to have models that are understandable or can {\bf explain} their predictions to humans \cite{kim2021machine}.

Explanations of a classification system can take many forms, but should accurately reflect the classifier's inner workings. Perhaps the best scenario is where the model itself is inherently understandable by humans. This is arguably true of decision trees, for instance. If the tree is small, then it can be fathomed in its entirety: a {\bf global explanation} of every prediction the model makes. If the tree is large, it can be hard to understand as a whole, but as long as it has modest depth, any individual prediction can be {\bf locally explained} using the features on the corresponding root-to-leaf path.

A common situation is where the predictive model is not inherently understandable, either at a global or local level, and so a separate {\bf post-hoc explanation} is needed. These are typically local, in the sense that they explain a specific prediction and perhaps also explain what the model does on other nearby instances. Over the past few years, many strategies for post-hoc explanation have emerged, such as LIME \cite{ribeiro2016should}, Anchors \cite{ribeiro2018anchors}, and SHAP \cite{lundberg2017unified}.

Explanation systems need to satisfy two broad criteria: the explanations should (i) make sense to a human user and (ii) be an accurate reflection of the actual predictive model. The first of these is hard to pin down because it is inextricably linked to vagaries of human cognition: is a linear model ``understandable'', for instance? Further research is needed to better characterize what (i) might mean. This paper focuses on criterion (ii): gauging the {\bf faithfulness} of explanations to the underlying predictive model, or put differently, the {\bf internal coherence} of the overall explanation system.

\subsection{Contributions}

We focus on classification problems and on explanation systems that consist of two components:
\begin{itemize}
\item A prediction function (the classifier) $f: \cX \to \cY$, where $\cX$ is the instance space and $\cY$ is the label space.
\item An explanation function $e: \cX \to \cE$, where $\cE$ is the space of explanations, or properties.
\end{itemize}
The explanation function explains the prediction $f(x)$ by pointing out some relevant property of the input. These properties can be quite general. Consider, for instance, a decision tree. Its prediction $f(x)$ on a point $x$ can be explained by the features on the root-to-leaf path for $x$; the explanation $e(x)$ is the conjunction of these features, e.g. ``$(x_2 > 0.5) \wedge (x_4 = \mbox{true}) \wedge (x_{10} < -1)$''. Thus the set $\cE$ has a conjunction for each leaf of the tree.

Or consider a classifier that takes an image $x$ of a landscape and returns its biome, e.g., {\tt rainforest}. One way this predictor $f(x)$ might operate is by identifying telltale flora or fauna in the image. For instance, if the image contains a zebra then its biome must be savannah: $f(x) = \mbox{\tt savannah}$ and $e(x) = \mbox{``contains a zebra''}$. Although such explanations are based on nontrivial attributes of the input, they are comprehensible to humans and are within the scope of our setup.

For an explanation system to be internally coherent, it should satisfy two properties:
\begin{itemize}
\item {\bf Consistency:} Roughly, two instances $x,x'$ that get the same explanation should also have the same prediction.

For instance, if two different images are assigned the same explanation, $e(x) = e(x') = \mbox{``contains a zebra''}$, then their assigned labels should also be the same.

\item {\bf Sufficiency:} If $x$ is assigned an explanation $e(x) = \pi$ that also holds for another instance $x'$ (even if $e(x') \neq \pi$), then $x'$ should have the same label as $x$.

For instance, if an image $x$ is assigned explanation $e(x) = \mbox{``contains a zebra''}$ and label $f(x) = \mbox{\tt savannah}$, then a different image $x'$ that also happens to contain a zebra should get the same label, even if it is assigned a different explanation, e.g. $e(x') = \mbox{``contains a baobab tree''}$.
\end{itemize}
These properties are desirable but might not hold in all cases. We introduce quantitative measures of the \emph{extent} to which they hold. 

With these measures in hand, we study a variety of established explanation systems: decision trees, Anchors, highlighted text, LIME,  SHAP, gradient-based method, $k$-nearest neighbors, and counterfactuals.
We show how they map into our framework and study their faithfulness. For instance, we prove that SHAP has perfect consistency, while LIME does not. 
We also have results at a higher level of abstraction. We formalize a natural sub-category of explanation systems that we call \emph{explicitly scoped rules}, that includes decision trees, anchors, and highlighted text. These have a common structure that permits their faithfulness to be studied in generality.

Another important use of these quantitative measures is to \emph{empirically} characterize the faithfulness of black-box explanation systems whose internals might not be known. We give statistical estimators for doing so and characterize their sample complexity. Along the way, we formalize what property a black-box explanation system should possess in order for its faithfulness to be easily verifiable. Roughly, this corresponds to a particular type of \emph{compression} achieved by the explanations. Indeed, we show (Claim~\ref{clm:unverifiable_explainer}) that absent any such compression, verification is not possible.

An interesting aspect of our measures is that the extent of faithfulness of an explanation system depends on the data distribution to which it will be applied, and thus might not be known at training time. Thus faithfulness may need to be assessed anew for each new setting in which the system will be used. In general, there is a tradeoff between simplicity of explanations and fidelity to the predictor. When explaining an animal recognizer, for instance, it might be reasonable to ignore special cases like marsupials if the system is used in North America, but not if it is used in Australia.

\vspace{-1em}
\paragraph{Summary of contributions:}
\begin{itemize}
\item Framework for evaluating the faithfulness of black-box explanation systems
\item Analysis of popular explanation methods
\item Estimators for faithfulness, with rates of convergence
\item Ease of estimation depends upon a notion of compression achieved by the explanations
\item Empirical evaluation of these measures and estimators
\item Highlighting fundamental properties of the faithfulness measure such as data dependence
\end{itemize}

\subsection{Related Work}
There are many types of explanation \cite{lipton2018mythos,molnar2019}. At a high level, we can separate them into two groups.
In \emph{intrinsic explanations}, 
the prediction models themselves are simple and self-explanatory, such as decision trees \cite{quinlan1986induction}, decision lists \cite{rivest1987learning}, and risk scores \cite{ustun2019learning}.
\emph{Post-hoc explanations} are applied to existing predictors and come in many varieties,
as described throughout the paper.

The importance of evaluating explanation methods has been discussed in the literature \cite{leavitt2020towards,zhou2021evaluating}.  
There are various attempts to measure different aspects of an explanation: usefulness to humans  \cite{jesus2021can,mohseni2018human,poursabzi2021manipulating}; complexity \cite{poppi2021revisiting}; difficulty of answering queries \cite{barcelo2020model}; and robustness \cite{alvarez2018robustness}. In this paper, we measure faithfulness to the model. Earlier work has looked at global measures of this type~\cite{wolf2019formal} and measures that are specialized to
neural networks \cite{poppi2021revisiting}, feature importance \cite{amparore2021trust,carmichael2021objective,sundararajan2017axiomatic,velmurugan2021developing}, rule-based explanations \cite{margot2020new}, surrogate explanation \cite{ribeiro2016should}, or highlighted text \cite{chen2018learning,wang2020towards,yoon2018invase}.  
\section{Framework}\label{sec-framework}
As described in the introduction, we think of an explanation system as consisting of a \emph{prediction function} (classifier) $f: \cX \to \cY$ and an \emph{explanation function} $e: \cX \to \cE$. (If $e(\cdot)$ is randomized, we can focus on one random seed.)
The local explanation for model $f$ at instance $x$ is some relevant property of $x$, denoted $e(x)$. The selected property should ideally be enough, on its own, to predict label $f(x)$. This general intuition has appeared in many places in the literature. Here we break it into two components---\emph{consistency} and \emph{sufficiency}---and provide precise measures of each.

\subsection{Consistency}\label{sec:consistency}

For any explanation $\pi \in \cE$, consider the set of instances that are assigned this explanation:
$$ C_\pi = \{x \in \cX: e(x) = \pi \}.$$
If $\pi$ is a good explanation, then we would hope that these instances all have the same predicted label. This is \emph{consistency}: instances that are assigned the same explanation should also be assigned the same prediction.

For some explanation systems, this may not hold all the time. We would like to quantify the extent to which it holds. We start by introducing a measure of the homogeneity of predictions in $C_\pi$. In order to do this, we need a distribution $\mu$ over instances $\cX$. This can be thought of as the distribution of instances that arise in practice.

\begin{definition}[local consistency]  The consistency of explainer $e$ for model $f$ at instance $x$, with respect to distribution $\mu$, is defined as 
$$m^c(x) = \Pr_{x'\in_\mu C_\pi}(f(x') = f(x))$$
where $\pi = e(x)$ and the notation $x' \in_\mu C_\pi$ means ``$x'$ is drawn from distribution $\mu$ restricted to the set $C_\pi$.''
\label{dfn:consistency}
\end{definition}

\paragraph{Global consistency.} We have so far quantified consistency at a specific instance. It is also of interest to measure the consistency of the \emph{entire} model. 

\begin{definition}[global consistency]
The global consistency of explanation system $(f,e)$, with respect to distribution $\mu$ over $\cX$, is  
$$m^c = \E_{x \in_\mu \cX} [m^c(x)].$$
\end{definition}

\paragraph{Relation to decoding.} The definition of global consistency implicitly defines a decoder $d$ from explanations $\pi$ to labels $y$. Recall that $C_\pi \subseteq \cX$ is the set of all instances that get assigned explanation $\pi$. These instances might not all have the same predicted label, but we can look at the distribution over labels,
$$ \Pr(y|\pi) = \Pr_{x \in_\mu C_\pi}(f(x) = y) .$$
With this in place, there are two natural ways to define the decoder: (i) the (randomized) \emph{Gibbs decoder} that, given explanation $\pi$, returns a label $y$ with probability $\Pr(y|\pi),$ and (ii) the optimal \emph{deterministic decoder} that returns the label $y$ that maximizes  $\Pr(y|\pi).$  We can denote the resulting decoding error, $\Pr_x(f(x)\neq d(e(x)))$, by $E_G$ for the Gibbs decoder and $E_O$ for the deterministic decoder. Standard manipulations show that these two errors are very similar:
\begin{claim}\label{clm:global_consis_gibbs_detr}
$E_O\leq E_G \leq 2E_O.$
\end{claim}

Our notion of consistency is exactly the accuracy of the Gibbs decoder: $m^c = 1-E_G$.

\subsection{Sufficiency} \label{sec:sufficiency}

A complementary requirement from an explainer is that if a property $\pi$ is used to justify the prediction at instance $x$, then any other instance $x'$ with property $\pi$ should also be classified the same way. Moreover, this should hold even if the supplied explanation, $e(x')$, is different from $\pi$. 
In the earlier biome example, if the explanation ``contains a zebra'' is ever used to justify a prediction of {\tt savannah}, then any picture with a zebra in it should get the same prediction, even if assigned some other property as explanation.

To start with, we say that explanations $\cE$ are \emph{intelligible} if for any instance $x \in \cX$ and property $\pi \in \cE$, it is possible to assess whether $\pi$ applies to $x$. If so, we define this as a relation $A(x,\pi)$. Ideally, the relation would not only be well-defined but would also be checkable by humans.
Note that the relation depends solely on the instance and not on the true or predicted label. 
We define the set of instances $C_x$ that share the same property as $x$'s explanation by $$C_x =\{ x' \in \cX: A(x',e(x)) \}.$$ 
As with the consistency measure, each instance can have a different level of sufficiency; it is not a binary value. To define this measure, we use a probability distribution over $C_x$. 
The sufficiency measure tests the homogeneity of predictions made on $C_x$.
\begin{definition}[local sufficiency]  The local sufficiency of explainer $e$ for model $f$ at instance $x$, with respect to distribution $\mu$, is defined as 
$$m^s(x) = \Pr_{x'\in_\mu C_x}(f(x') = f(x)).$$
Recall that the notation $x' \in_\mu C_x$ means ``$x'$ is drawn from distribution $\mu$ restricted to the set $C_x$''.
\label{dfn:sufficiency}
\end{definition}

Consistency and sufficiency are complementary measures.
For any given instance $x$, $m^c(x)$ can be larger, smaller, or equal to $m^s(x).$ 
Similarly to global consistency, we define a sufficiency measure for the entire model.  
\begin{definition}[Global sufficiency]
The global sufficiency of explanation system $(f,e)$, with respect to distribution $\mu$ on $\cX$, is equal to $m^s = \E_{x\in_\mu\cX}[m^s(x)].$
\end{definition}

\section{Analysis of common explanation systems}

In this section we review some popular explanation methods and assess the extent to which they achieve consistency and sufficiency. We divide these methods into three sub-categories: explicitly scoped rules, feature importance scores, and example-based explanations.

\subsection{Explicitly scoped rules}

In ``scoped rules", each explanation is an explicit region of the instance space, e.g., ``$(x_2 > 0.5) \wedge (x_4 = \mbox{true}) \wedge (x_{10} < -1)$''. This type of explanation includes decision trees, anchors, and highlighted text, which we elaborate on next.

\subsubsection{Decision trees}\label{sec:decision-trees}

Suppose the instance space is some $\cX \subseteq \R^d$. When a decision tree is used to ``explain'' a classifier $f: \cX \to \cY$, the tree is fit to $f$'s predictions \cite{dasgupta2020explainable,NEURIPS2019_ac52c626,moshkovitz2021connecting}. The explanation of an instance $x$ is the conjunction of the features along the path from $T$'s root to the leaf in which $x$ lies. Thus the explanations, $\cE$, are in one-to-one correspondence with the leaves of the tree.

In this case, an explanation $\pi$ applies to an instance $x$ if and only if $x$ falls in $\pi$'s leaf. Therefore, the relation $A(x,\pi)$ is intelligible (well-defined) and easy for a human to assess. Moreover, consistency is equal to sufficiency, and they measure the accuracy of the tree in capturing $f$:
\begin{align}
m^c &= m^s 
= \Pr_{X,X' \sim \mu} \left( f(X) = f(X') | e(X) = e(X') \right)  \nonumber  \\
&= \sum_{\mbox{leaves $\pi$}} \mu(C_\pi) \Pr(f(X) = f(X')|X,X' \in C_\pi)  \nonumber \\
&= 1 - \sum_{\pi} \mu(C_\pi) (\mbox{\rm Gini-index of $f$ in }C_\pi)  \nonumber
\end{align}
where $C_\pi$ is the subset of $\cX$ that ends up in leaf $\pi$.

\subsubsection{Anchors}\label{sec:anchors}

Pick any data space $\cX \subseteq \R^d$ and prediction function $f: \cX \to \cY$. An \emph{anchor} explanation \cite{ribeiro2018anchors} for an instance $x \in \cX$ is an explicitly-specified hyperrectangle $H_x \subset \R^d$ that contains $x$ and that is meant to correspond, roughly, to a region around $x$ that is similarly labeled. 

The quality of an anchor is typically formalized using the notion of \emph{precision}, which is the probability, over the distribution $\mu$, that a random instance in $H_x$ has label $f(x)$, that is, $\Pr_{x'\in_\mu H_x}[f(x') = f(x)]$.

In this case, the space of explanations is the set of all anchor-hyperrectangles, $\cE = \{H_x: x \in \cX\}$. It is easy to check whether an anchor applies to an instance: the relation $A(x, H) \equiv (x \in H)$ is well-defined. Moreover, our notion of local sufficiency is exactly the precision of anchors and global sufficiency is exactly the average precision of anchors:
\begin{align*}
m^s &= \Pr_{X,X' \sim \mu}(f(X') = f(X)|A(X', e(X)))\\
&= \E_{X \sim \mu}[\mbox{precision}(H_X)] .
\end{align*}

If anchors are chosen to be discrete---that is, the same hyperrectangles are used many times---then our notion of consistency gauges the uniformity of prediction over all instances for which a particular anchor is specified:
$$ m^c = \Pr_{X, X' \sim \mu}(f(X') = f(X)|H_{X'} = H_X) .$$
There is no immediate relation between this and sufficiency or precision.

\subsubsection{Highlighted text}

The goal in \emph{highlighted text} explanations
is to pick out the features---for instance, words in text---that are most important for a model's prediction \cite{jacovi2021aligning}. For an instance $x \in \mathbb{R}^d$, the explanation can be thought of as a subset of features $S\subseteq [d]$, and the values (e.g., text) of these features, $x_S \in \mathbb{R}^{|S|}$.

These explanations are anchors at the level of generality of Section~\ref{sec:anchors}. Thus the same observations apply here.

\subsubsection{A unified framework for explicitly scoped rules}\label{sec:scoped-rules}

The last three examples---decision trees, anchors, and highlighted text---have a common structure that is appealingly simple and may also hold for many future explanation systems. To formalize it, we say an explanation system $(f,e)$ has \emph{explicitly scoped rules} if each explanation $\pi$ is a description of a region $S_\pi \subseteq \cX$ of the instance space. For a given point $x$, the explanation $\pi = e(x)$ has the property that $x \in S_\pi$. The terminology ``explicitly scoped'' means that subset $S_\pi$ is specified in a form where it is easy to check whether a specific point lies in it or not. Thus the set of explanations is $\cE = \{e(x): x \in \cX\}$ and the relation $A(x,\pi) \equiv (x \in S_\pi)$ is well-defined (intelligible). This is the key property of explicitly-scoped rules.

We can generalize the notion of precision to any region of space (not just hyperrectangles) and as in the case of anchors, sufficiency will then correspond to average precision. 

For the explanation systems we will cover next, intelligibility---determining whether an explanation applies to a given instance---is more tricky.

\subsection{Feature importance methods}

\emph{Feature importance} methods aim to give a precise indication of which features of an input $x$ are most relevant to the prediction $f(x)$. This often takes the form of a local linear model $g_x$ (sometimes on a simplified instance space) that approximates $f$ in the vicinity of $x$. However, the scope of this $g_x$---the region over which the approximation is accurate---is sometimes unspecified, in which case it is unclear when a particular $g_x$ can be thought of as being applicable to some other point $x'$. Because of the ambiguity in the intelligibility of these explanations, we will focus on consistency in what follows.

\subsubsection{LIME}

LIME \cite{ribeiro2016should} provides an explanation of $f(x)$ by (1) using an \emph{interpretable representation} $\psi:\cX\rightarrow\cX'$, e.g. the presence or absence of individual words in a document, and (2) approximating $f$ near $x$ with a simple model $g_x: \cX' \to \cY$. Typically, $g_x$ is a linear classifier.

LIME does not exhibit perfect consistency, e.g., points $x$ with the same interpretable representation get assigned the same $g_x$, while their predicted labels may vary. Another example is depicted in Appendix~\ref{apx:lime_no_perfect_consistency}.

\subsubsection{SHAP}

SHAP \cite{lundberg2017unified} is similar in spirit to LIME. It uses a Boolean feature space $\cX'$ and its explanations are linear functions $g_x: \cX' \to \cY$. But this time the choice of $g_x$ is inspired by Shapley values \cite{shapley1953value} from game theory, and is chosen to satisfy four axioms for fair distribution of gains: efficiency, symmetry, linearity, and null player. In particular, the coefficients of $g_x$ are guaranteed to sum to $f(x) -\phi_0$, where $\phi_0$ is constant for all $x$.

This last property guarantees that if two examples have the same explanation, then their label must be the same, thus ensuring perfect consistency. 

\subsubsection{Gradient-based method}

Gradient-based explanations are popular for neural nets \cite{agarwal2021towards,ancona2017towards,shrikumar2016not,simonyan2013deep,smilkov2017smoothgrad}. The explanation is the gradient of the network with respect to the instance, the intuition being that features with highest gradient values have the most influence on the model's output.

The gradient alone determines a function only up to an additive constant. This offset must also be provided to complete the explanation; otherwise there is imperfect consistency. The lack of decodability was empirically observed in several previous works \cite{adebayo2018sanity,anders2020fairwashing,kim2021machine,nie2018theoretical,wang2020gradient}.

\subsection{Example-based explanations}

An example-based explainer justifies the prediction on an instance $x$ by returning instances related to $x$. Explanations of this type include \emph{nearest neighbors} and \emph{counterfactuals}. 

\subsubsection{Nearest neighbors}

Let's focus on 1-nearest neighbor for concreteness. For a given prediction function $f: \cX \to \cY$, a nearest neighbor explanation system maintains a set of prototypical instances $\mathcal{P} \subseteq \cX$ and justifies the prediction on instance $x$ by returning a prototype $p \in \mathcal{P}$ close to $x$ (with respect to an underlying distance function $d$ on $\cX$). Thus the space of explanations is $\cE = \mathcal{P}$.

Our consistency measure then checks the extent to which points $x,x'$ that get mapped to the same prototype $p \in \mathcal{P}$ also get the same prediction under $f$. 

For sufficiency, we also need to define the relation $A(x,p)$: when do we consider prototype $p$ to be ``applicable to'' instance $x$? Here are two options.
\begin{enumerate}
\item When $p$ is the nearest neighbor of $x$ in $\mathcal{P}$.
\item When $d(x,p) \leq \tau$ for some threshold $\tau > 0$.
\end{enumerate}
The first option strictly follows the nearest neighbor rule, but leads to problems with verifiability; for instance, it is not easy for a human to check that $A(x,p)$ holds unless the set $\mathcal{P}$ is somehow available. The second option is easier to check; in fact, we can treat the regions $B(p,\tau)$ as anchors and then measure consistency and sufficiency using the methods of Section~\ref{sec:scoped-rules}. 

\subsubsection{Counterfactuals}

A counterfactual explanation of an instance $x$ is another instance $x'$ which is close to $x$ but has a different label, $f(x) \neq f(x')$ \cite{ deutch2019constraints,mothilal2020dice,slack2021counterfactual}. To make this concrete, suppose we are performing binary classification and that some distance function $d$ has been chosen for the instance space $\cX$. Then the counterfactual explanation for $x$ is the closest point $x'$ that gets the opposite label, that is, $x'=\argmin_{x': f(x')\neq f(x)} d(x,x')$. The space of explanations is $\cE = \cX$.

In this case, the ``explanation'' $x'$ gives information about the nature of predictions in the vicinity of $x$. Specifically, it asserts that \emph{any point in the open ball $B(x, d(x,x'))$ has label $f(x)$}. Therefore, one way to verify faithfulness of these explanations is simply to associate them with scoped rules of this form and to then assess sufficiency as in Section~\ref{sec:scoped-rules}.

\section{Evaluating faithfulness of black-box model}\label{sec:evaluating_faithfulness_model}

In this section, we consider a scenario where we are given a black-box explanation system $(f,e)$ and wish to evaluate its faithfulness. To this end, we develop statistical estimators for consistency and sufficiency given samples $x_1, \ldots, x_n$ from an underlying test distribution $\mu$. How many such samples are needed to accurately assess faithfulness?

\subsection{Discrete explanation spaces}

Let's begin with the case where the explanation space $\cE$ is discrete (that is, countable). We will not assume that we know the entire set $\cE$, since this knowledge will not be in general available for a black-box explanation system. Given a few samples from $\mu$, we can look at the resulting explanations and predictions, but it is not trivial to assess the fraction of the explanation space that we have not seen: that is, the \emph{missing mass}. And for any explanation $\pi$ that we do not see, faithfulness could be arbitrarily bad. With this difficulty in mind, we now turn to our estimators.

A key observation is that although consistency and sufficiency measure different aspects of the explanation system, for the purposes of statistical estimation they can be treated together. To see this, let $R(x,\pi)$ denote an arbitrary relation on $\cX \times \cE$, and for a given distribution $\mu$ on $\cX$, define
$$ m^{R}_\mu = \Pr_{X,X' \sim \mu} \left( f(X') = f(X) | R(X', e(X)) \right) .$$
This generalizes both types of faithfulness: for consistency, take $R(x,\pi)$ to mean $e(x) = \pi$ and for sufficiency take $R(x,\pi) \equiv A(x,\pi)$.

Thus we only need an estimator for $m^R_\mu$. The quality of our estimate will depend upon what fraction of the explanation space we get to see, which in turn depends on $\cE$ and $\mu$. 

We begin with a few related definitions. Let $p(\pi)$ be the fraction of points for which explanation $\pi$ is provided, that is, $p(\pi) = \mu(\{x: e(x) = \pi\})$ and let $q(\pi)$ be the fraction for which $R(x,\pi)$ holds: $q(\pi) = \mu(\{x: R(x,\pi)\})$. Thus $p(\pi)$ is a distribution over $\cE$ while $q(\pi)\in [0,1]$ and $q(\pi) \geq p(\pi)$.

Given samples $x_1, \ldots, x_n \sim \mu$, and any $y, \pi$, define
\begin{align*}
    N_\pi &= |\{i: R(x_i, \pi)\}| \\
    N_{\pi,y} &= |\{i: R(x_i, \pi), f(x_i) = y\}|
\end{align*}
Our estimator for $m^R_\mu$ is then
$$ \widehat{M} = \frac{1}{n} \sum_{i=1}^n {\bf 1}(N_{e(x_i)} > 1) \, \frac{N_{e(x_i), f(x_i)} - 1}{N_{e(x_i)} - 1} .$$

We can show the following rate of convergence.
\begin{theorem}
The mean-squared error of estimator $\widehat{M}$ can be bounded as follows:
$$ \E \left[ (\widehat{M} - m^R_\mu)^2 \right] \ \leq \ \frac{4}{n} + \left( \sum_\pi p(\pi) e^{-(n-1)q(\pi)} \right)^2 .$$ 
\label{thm:rate-of-convergence}
\end{theorem}

The mean-squared error is the sum of the variance, which is bounded by $4/n$, and the squared bias, the term in parentheses. This bias arises from the inability to correctly assess faithfulness for explanations $\pi$ that appear $0$ or $1$ times in the data. One way to make this term small, say $<\epsilon$, is to have $n$ comparable to the size of $\mathtt{range}(e)$.
In this way, we see that the ease of evaluating the faithfulness of explanations depends on the level of \emph{compression} they achieve.

\begin{corollary}
Suppose the unlabeled sample size is at least $n\geq \mathtt{range}(e)\cdot\frac{12}\epsilon\log\frac3\epsilon$. Then, the mean-squared error of $\widehat{M}$ (for either consistency or sufficiency) is at most $\epsilon$.
\label{cor:estimator-sample-complexity}
\end{corollary}

\subsection{Larger or continuous explanation spaces}

The estimator of the previous section needs explanations to appear at least twice before it can begin assessing their faithfulness. This is problematic in continuous explanation spaces, where no explanation might ever be repeated.

One fix, which we later study empirically, is to discretize the space $\cE$. We introduce a function $\psi: \cE \to \cE'$ where $\cE'$ is much smaller than $\cE$, and consider explanations $\pi, \pi'$ to be equivalent if $\psi(\pi) = \psi(\pi')$. An alternative fix is to introduce a distance function $d$ between explanations, and to use a $d(\pi, \pi') \leq \tau$ to determine when $\pi$ and $\pi'$ are close enough that they should yield the same prediction.

Explainers that their inner-working is known, their consistency and sufficiency might also be known (e.g., SHAP has perfect consistency). However, the new measures need to be estimated if the inner working is unknown. This section provided conditions where such an estimation is possible. Unfortunately, there are some cases where it is impossible to apply \emph{any} estimation method. One such scenario is where all the explanations are distinct, as the next claim shows.

\begin{restatable}{claim}{clmUnverifiable}(unverifiable explainer)\label{clm:unverifiable_explainer}
Fix infinite example set $\cX$. There are two explainers, $e_1$ and $e_2$, a model $f$, and a distribution over the examples, where on every finite-sample, with probability $1,$ the explanations are the same, but the sufficiency and consistency of $e_1$ is $1$ while the sufficiency and consistency of $e_2$ is $0.5$. 
\end{restatable}

\section{Experiments}

\subsection{Canonical properties}\label{subsec:canonical_properties}

We begin with experiments that illustrate basic properties of our faithfulness estimators: 
(1) they assign low scores to random explanations, (2) they assign higher scores to more faithful explanations as long as the explanation space is not too large, and (3) when the explanation space is huge relative to the amount of unlabeled data, they conservatively assign low scores since they are unable to assess faithfulness.

\paragraph{Highlighted text.}

To evaluate a variety of \emph{highlighted text} explainers, 
we began by training a predictor on the \polaritydata{} dataset, used for sentiment classification of movie reviews, with 10,433 documents. We represented each document as a bag of words, and used 80\% of the data to train a linear model. The remaining documents were used to compare four highlighted text explainers. 

We evaluated four explainers. (1) \emph{Top Coefficient} is a white-box explainer that highlights the word in the sentence with the highest absolute coefficient in the linear model. (2) \emph{Anchors} \cite{ribeiro2018anchors}. (3) \emph{First Word} always highlights the first word in the sentence as the explanation.
(4) \emph{All Words} highlights all the words in the sentence as the explanation. We estimated global consistency and sufficiency for each explainer as described in previous sections. We also recorded the \emph{uniqueness} of each explainer, which is the fraction of test data whose explanations were unique.

The results are presented in Table~\ref{tbl:word_highlight_sanity}.
Top Coefficient 
got the highest consistency and sufficiency scores,
as one might expect from an explainer that utilizes its complete knowledge of the model. As Anchors is a black-box explainer that attempts to return a faithful explanation, it produces better results than the last two explainers, which are not designed to be faithful to the model.
First Word is close to a random explainer, and thus gets rather low sufficiency and consistency. All Words highlights the entire input and thus has maximal uniqueness ($1.0$), making it unverifiable. Consequently, its consistency and sufficiency estimates are $0.0$, despite the definitions implying a value of $1.0$ for both measures.

\begin{table}[!htb]
    \centering
    \caption{
    The mean$\pm$std of the consistency, sufficiency, and uniqueness measures of the four highlighted text explainers, evaluated over 5 samples of 1000 examples.
    }
    \label{tbl:word_highlight_sanity}
    \begin{tabular}{l l l l}
         \toprule
         Explainer & Consistency & Sufficiency &  Uniqueness\\
         \midrule
         Top Coefficient & \textbf{0.69} $\pm$ 0.01 & \textbf{0.71} $\pm$ 0.01  & 0.5 $\pm$ 0.01 \\
         Anchors & 0.54 $\pm$ 0.01 & 0.61 $\pm$ 0.01 & 0.44 $\pm$ 0.01 \\
         First Word & 0.37 $\pm$ 0.01 & 0.48 $\pm$ 0.01 & 0.39 $\pm$ 0.01 \\
         All Words & 0.0 $\pm$ 0.0 & 0.0 $\pm$ 0.0 & 1.0 $\pm$ 0.0 \\
         \bottomrule
    \end{tabular}
\end{table}

\paragraph{Decision trees.}
Next, we used decision trees to study the relationship between the size of the explanation space and the number of samples needed for accurate estimation of faithfulness. Each of our prediction models was a decision tree, and the same tree was used for explanations, implying perfect consistency and sufficiency.
We learned six trees of different sizes ($2^n$ leaves, for $n = 6, 7, \ldots, 11$) on the \adult{} dataset~\cite{kohavi1996scaling}, using $66.6\%$ of the examples for training. From the remaining $33.3\%$ of the examples we varied the number of sampled records used to estimate consistency/sufficiency (the two estimates are identical in this setting).

The results appear in Figure~\ref{fig:main_adult_decision_tree}.
For the smallest tree (64 leaves), the global estimator is accurate even with very few samples. However, as the size of the tree grows, there are more possible explanations (root-to-leaf paths), which increases the sample complexity of the estimation task. For example, the largest configuration (2048 leaves) requires 4,300 samples to reach even a 0.9 estimate of sufficiency and consistency. Similar trends were observed for different datasets and when k-nearest neighbors was used as both the model and explainer (Appendix~\ref{sec:appendix_sample_complexity}).

\begin{figure}[!htb]
    \centering
    \includegraphics[width=.8\linewidth]{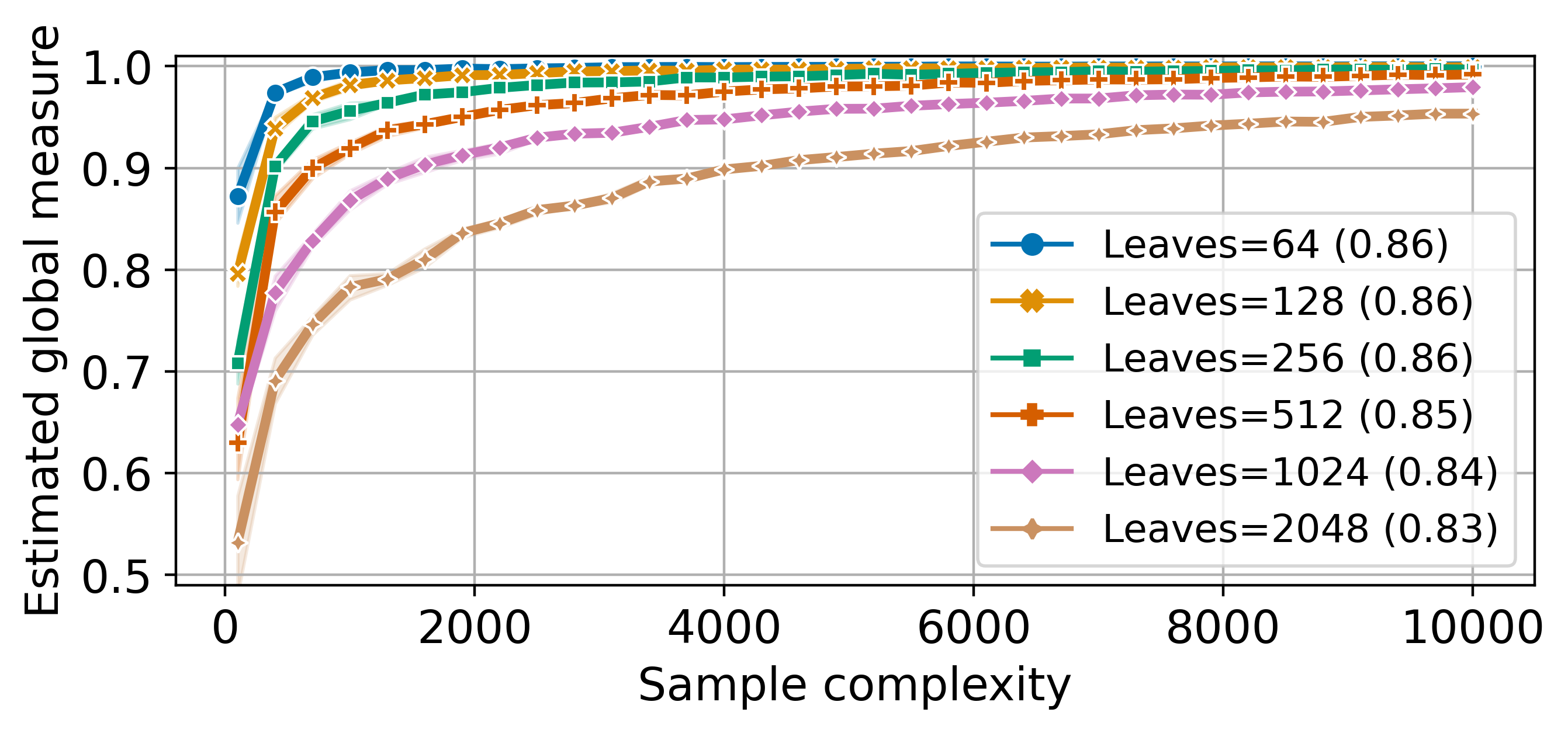}
    \caption{Estimated consistency and sufficiency of decision trees of different sizes over the \adult{} dataset, as a function of the sample complexity.
    As the sample complexity increases the estimation approaches the ground truth measures (1.0). Larger trees have more leaves and thus a larger explanation space. The decision tree model accuracy over the full test set is reported in the parenthesis in the legend.
    The displayed results are averaged over 5 executions with a confidence interval of 95\%.}\label{fig:main_adult_decision_tree}
\end{figure}

\subsection{Common explanation systems}\label{subsec:exp_common_explainers}

We next discuss two important considerations in applying our faithfulness estimators in practice: (1) the effect of explainer parameters on the quality of the estimates, and (2) the use of discretization to reduce explanation uniqueness and thereby improve estimation. These apply generically for many common explanation systems. For concreteness, the predictors in our experiments are gradient boosted trees, which are frequently used by practitioners (described in Appendix~\ref{sec:appendix_traning}). The analysis is conducted on six standard datasets (described in Appendix~\ref{sec:appendix_datasets}).

First, the choice of the explainer's parameters 
can impact not only sufficiency and consistency but also the accuracy of estimation. 
For example, a key parameter in Anchors is precision \threshold{}, and high \threshold{} leads to better sufficiency. Moreover, for high \threshold{}, the anchors are typically smaller, as more explanations are possible (increasing the number of explanations worsens the estimators, as seen in Section~\ref{subsec:canonical_properties}). We next illustrate this phenomenon both locally and globally over the \adult{} dataset.  

Figure~\ref{fig:adult_anchor_example}, shows an example of the effect on local measures.
(\ref{fig:adult_anchor}) shows the explanations, $\pi_1$ and $\pi_2$, of two different anchors over the record depicted in (\ref{fig:adult_record}).
The two explainers differ only in their precision \threshold{} parameter (0.5 and 0.95).
(\ref{fig:adult_anchor_stats}) presents the statistics of these explainers when applied over the record from (\ref{fig:adult_record}) and the \adult{} test set.   
One can see that $\pi_2$ refines $\pi_1$ since it includes more conditions, and hence $\lvert C_{\pi_1} \rvert > \lvert C_{\pi_2} \rvert$.  Moreover, using higher \threshold{} improved the sufficiency. 

\begin{figure}[!htb]
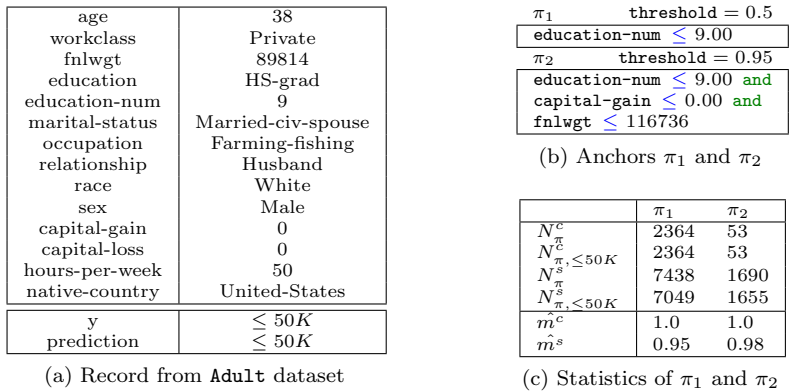

    \centering
    \begin{minipage}{.55\linewidth}
    \centering
    \subfloat[Record from \adult{} dataset]{
        \centering
        \scriptsize
        \begin{tabular}{|c|c|}
            \hline
             age & 38 \\
             workclass & Private  \\
             fnlwgt & 89814 \\
             education & HS-grad \\
             education-num & 9 \\
             marital-status & Married-civ-spouse \\
             occupation & Farming-fishing \\
             relationship & Husband \\
             race & White \\
             sex & Male \\
             capital-gain & 0 \\
             capital-loss & 0 \\
             hours-per-week & 50 \\
             native-country & United-States \\
             \hline
             \hline
             y & $\leq 50K$ \\
             prediction & $\leq 50K$ \\
             \hline
        \end{tabular}
        \label{fig:adult_record}
    }
    \end{minipage}%
    \begin{minipage}{.45\linewidth}
    \centering
    \subfloat[Anchors $\pi_1$ and $\pi_2$]{
        \centering
        \scriptsize
        \begin{tabular}{l}
            $\pi_1$ \hfill $\threshold{}=0.5$\\
            \hline
            \multicolumn{1}{|l|}
            {$\texttt{education-num } {\textcolor{ao}{\leq}} \text{ } 9.00 $}\\
            \hline
            $\pi_2$ \hfill $\threshold{}=0.95$\\
            \hline
            \multicolumn{1}{|l|}
            {$\texttt{education-num } {\textcolor{ao}{\leq}} \text{ } 9.00 {\textcolor{ao(english)}{\texttt{ and}}}$} \\
            \multicolumn{1}{|l|}
            {$\texttt{capital-gain } {\textcolor{ao}{\leq}}  \text{ } 0.00 {\textcolor{ao(english)}{\texttt{ and}}}$} \\
            \multicolumn{1}{|l|}
            {$\texttt{fnlwgt } {\textcolor{ao}{\leq}}  \text{ } 116736$}\\
            \hline
        \end{tabular}
        \label{fig:adult_anchor}
    }\\
    \subfloat[Statistics of $\pi_1$ and $\pi_2$]{
    \scriptsize
    \centering
    \begin{tabular}{|l|l l|}
        \hline
        & \textbf{$\pi_1$} & \textbf{$\pi_2$} \\
        \hline
        $N^{c}_{\pi}$ & 2364 & 53 \\
        $N^{c}_{\pi, \leq 50K}$ & 2364 & 53 \\
        $N^{s}_{\pi}$ & 7438  & 1690\\
        $N^{s}_{\pi, \leq 50K}$ & 7049 & 1655 \\
        \hline
        $\hat{m^c}$ & 1.0 & 1.0 \\
        $\hat{m^s}$ & 0.95 & 0.98 \\
        \hline
    \end{tabular}
    \label{fig:adult_anchor_stats}
    }
    \end{minipage}
    \caption{Two Anchors explainers with different precision \threshold{} parameters, and their performance statistics over an example record from the \adult{} dataset.}
    \label{fig:adult_anchor_example}
\end{figure}

Moving to global measures, 
Figure~\ref{fig:adult_anchor_estimates} shows faithfulness estimates for Anchors applied to gradient boosted trees trained on the \adult{} dataset (Appendix~\ref{sec:appendix_anchors_precision} has results for other datasets), as a function of the precision. 
As expected, as the precision increases, so do sufficiency and uniqueness. Note that higher uniqueness reduces estimator accuracy.

\begin{figure}
    \centering
    \includegraphics[width=.8\linewidth]{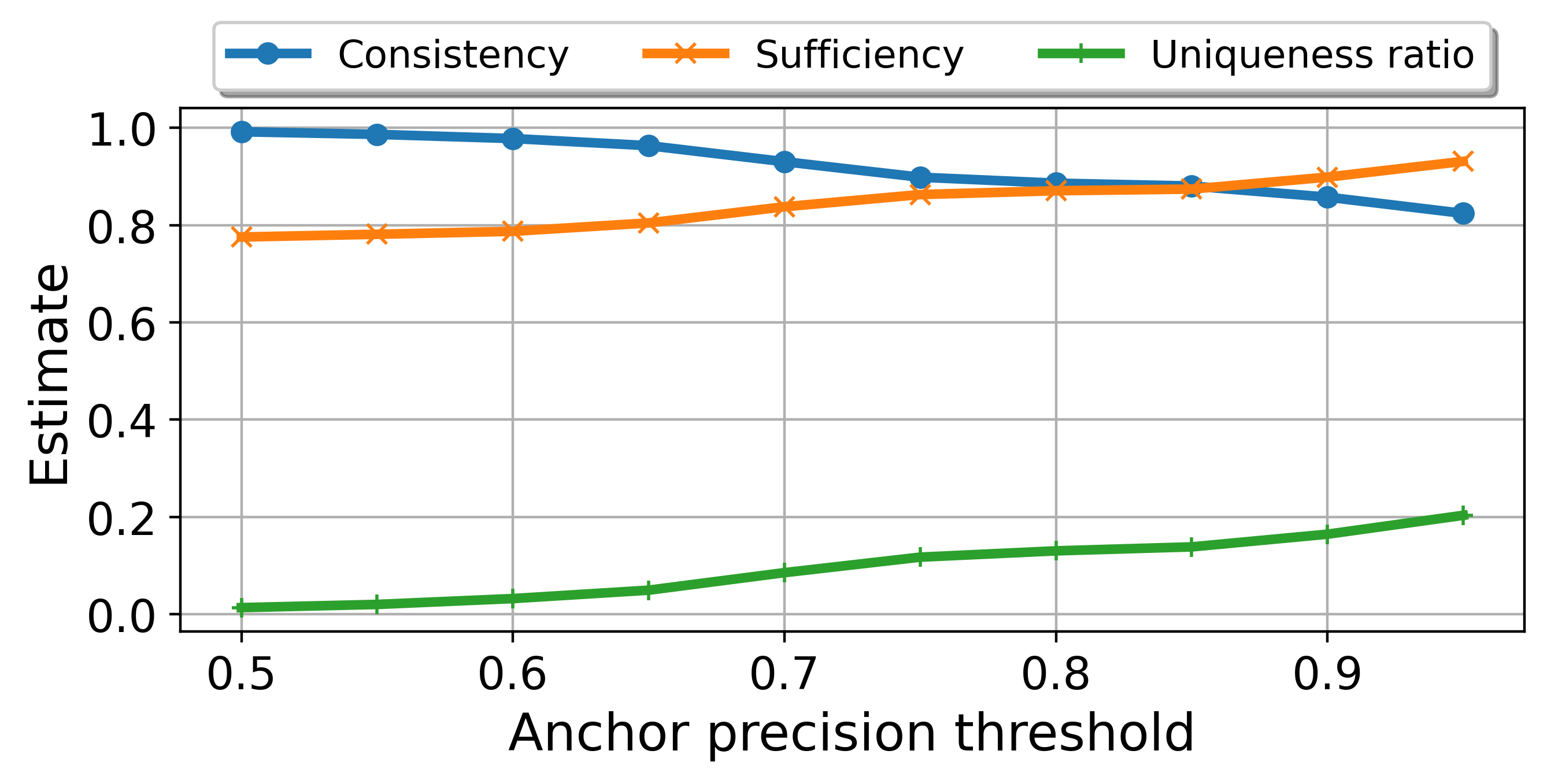}
    \caption{Estimated global consistency and sufficiency and the number of unique explanations of the Anchors explainer over gradient boosted trees model for the \adult{} dataset as a function of precision \threshold{} parameter.}
    \label{fig:adult_anchor_estimates}
\end{figure}

Second, discretizing the output seems to be an effective way to mitigate uniqueness.
This is illustrated in Table~\ref{tbl:SHAP_discretization}, which shows the results of 5 different discretization methods of SHAP values (described in Appendix~\ref{sec:appendix_discretization}) and a non-discretized baseline over 6 datasets. 
As one would expect (based on Claim~\ref{clm:unverifiable_explainer}), without discretization the measures are extremely low. Moreover, while all examined discretization methods improve on the non-discretized baselines, the optimal method depends on the dataset and explainer at hand. Hence, one is encouraged to experiment with different methods to identify the best approach.
In Appendix~\ref{sec:appendix_discretization} we also provide the uniqueness ratio of each discretization method, along with discretizations of LIME and Counterfactuals that exhibit similar behavior.

\begin{table}[!htb]
    \centering
    \caption{SHAP consistency scores for various discretizations, averaged over 5 executions (std is lower than 0.01 in all cases).}
    \label{tab:shap_disc}
    \begin{tabular}{c c c c c c c c}
        \toprule
        Dataset &
        Original &
        2-FP &
        1-FP & 
        Sign &
        Rank &
        Sign-of-top-5 \\
        \midrule
        Heart & 0.0 & 0.0 & \textbf{0.48} & 0.02 & 0.02 & 0.39 \\
        Chess & 0.0 & 0.0 & 0.0 & 0.33 & 0.32 & \textbf{0.35} \\
        Avila & 0.01 & 0.01 & 0.05 & \textbf{0.71} & 0.56 & 0.58 \\
        Bank marketing & 0.03 & 0.40 & \textbf{0.93} & 0.49 & 0.38 & 0.86 \\
        Adult & 0.02 & 0.11 & \textbf{0.95} & 0.68 & 0.15 & 0.89 \\
        Covtype & 0.01 & 0.03 & \textbf{0.68} & 0.13 & 0.09 & 0.41 \\
        \bottomrule
    \end{tabular}\label{tbl:SHAP_discretization}
\end{table}

\subsection{Explanation quality is data-dependent}\label{subsec:exp_data_dependent}
The consistency and sufficiency definitions imply that the faithfulness of an explainer depends on the test distribution. When two explanation methods are available, one might be more faithful for some populations, while the other works better for other populations. Moreover, as data distribution changes over time, explanation methods must also adapt. It is not advisable to deploy explainers in real-life settings without verifying faithfulness on the target distribution.

In Figure~\ref{fig:populations_adult_consistency} we demonstrate this by splitting the \adult{} dataset into two different populations. Each positive example is in the first population with probability $0.75$ and each negative example is in the first population with probability $0.25$. We used $4$ different explainers on the two populations (Anchors with \threshold{} of 0.7, SHAP and LIME with 1-FP discretization, and Counterfactuals with discretization of the sign of modification). We estimated the consistency of each population for $5$ repetitions and recorded the average and standard deviation. For all the explainers, the consistency is different between the two populations. We remark that Anchors has the highest consistency in the first population ($0.934\pm0.003$), while SHAP has the highest consistency in the second  ($0.885 \pm 0.006$).

\begin{figure}[!htb]
    \centering
    \includegraphics[width=.8\linewidth]{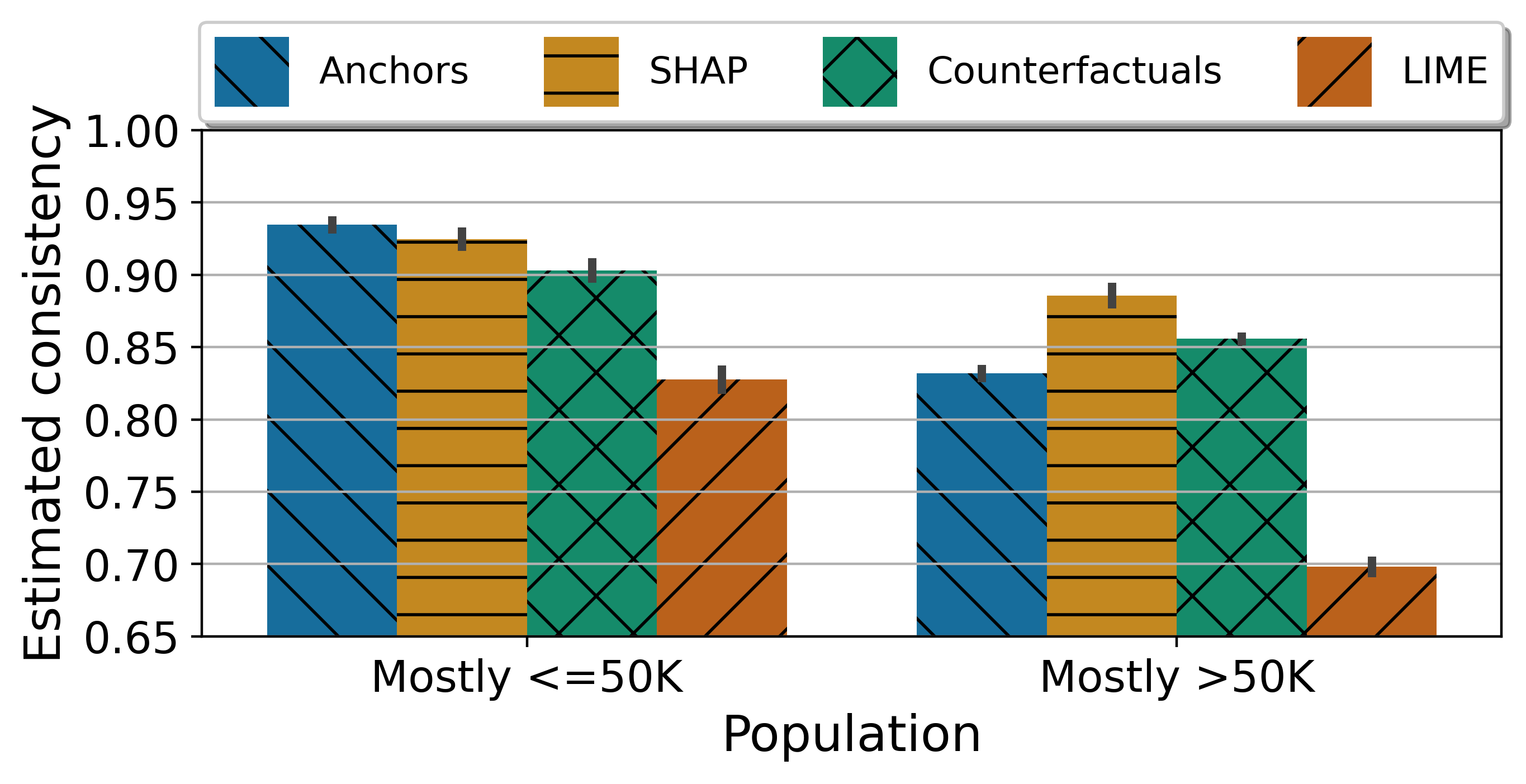}
    \caption{
    Estimated global consistency differentiate over distributions of \adult{} dataset.
    Each of the test examples was randomly assigned to one of two populations based on its label. Example with label ``$\leq$50K'' (resp. ``$>$50K'') chances of being assigned to the first population are 0.75 (resp. 0.25) and chances to be assigned to the second are 0.25 (resp. 0.75).
    The displayed results are the mean of 5 executions with std bars.}
    \label{fig:populations_adult_consistency}
\end{figure}

\section{Conclusion}
We suggest two new measures evaluating the faithfulness of explanations, both locally and globally. These are the consistency and sufficiency measures. We showed estimators for these measures and bounded the sample complexity of the global measures by an \emph{unlabeled} sample of size $O(\mathtt{range}(e))$, for constant $\epsilon$ error. We analyzed these measures on several known methods: decision trees, Anchors, highlighted text, SHAP, LIME, gradient-based method, $k$-nn, and counterfactuals. 
We empirically examined these measures, highlighting essential properties, e.g., faithfulness can be unverifiable if there are too many explanations and faithfulness quality is data-dependent.

\section*{Acknowledgements}
We would like to thank Yoav Goldberg for introducing us to the problem of faithfulness in NLP which initiated this project. 

\section*{Funding transparency statement}
Sanjoy Dasgupta has been supported by NSF CCF-1813160 and  NSF IIS-1956339. Nave Frost has been funded by the European Research Council (ERC) under the European Union’s Horizon
2020 research and innovation programme (Grant agreement No. 804302).
Michal Moshkovitz has received funding from the European Research Council (ERC) under the European Union’s Horizon 2020 research and innovation program (grant agreement No. 882396), by the Israel Science Foundation (grant number 993/17), Tel Aviv University Center for AI and Data Science (TAD), and the Yandex Initiative for Machine Learning at Tel Aviv University.

\bibliography{main}
\bibliographystyle{plain}

\newpage
\appendix
\onecolumn

\section{Proofs}

\subsection{Proof of Claim~\ref{clm:global_consis_gibbs_detr}}

For the sake of completeness, we repeat some of the definitions. For a fixed explanation $\pi$, the probability that it resulted from an instance labeled $y$ is equal to $$\Pr(y|\pi)=\sum_{x:f(x)=y\wedge e(x)=\pi}\frac{\Pr(x)}{\Pr(\pi)},$$ where $\Pr(\pi) = \sum_{x:e(x)=\pi} \Pr(x)$ and the distribution over the instances $x$'s is $\mu.$
There are two natural ways to define decoders from explanations to labels:
\begin{itemize}
    \item Gibbs decoder 
    \end{itemize}
    $$d_G(\pi) = y \text{ with probability }  \Pr(y|\pi)$$ 
\begin{itemize}
\item Optimal deterministic decoder $$d_O(\pi) =\argmax_y \Pr(y|\pi)$$ 
\end{itemize}

The error of any decoder $d$ is equal to 
$$\sum_{\pi}\Pr(\pi)\sum_{y\in\cY}\Pr(y|\pi)\Pr(d(\pi)\neq y |\pi).$$
Specifically, the error of the Gibbs decoder is equal to 
$$E_G = \sum_{\pi}\Pr(\pi)\sum_{y\in\cY}\Pr(y|\pi)(1-\Pr(y|\pi)).$$
The error of the optimal deterministic decoder is equal to 
$$E_O=\sum_\pi\Pr(\pi)(1-\max_y\Pr(y|\pi)).$$
Now we are ready to prove the claim that $$E_O\leq E_G \leq 2E_O.$$ 
For ease of notations, arrange the probabilities $(\Pr(y|\pi))_{y\in\cY}$ in decreasing order $p_1\geq p_2\geq\ldots p_{|\cY|}.$

We first prove the left inequality in the claim. We will show that for every explanation $\pi$ it holds that $$1-p_1\leq \sum_j p_j(1-p_j).$$ Or equivalently, we will show that $$\sum_j p_j^2\leq p_1.$$ The latter holds because 
$$\sum_j p_j^2 \leq \sum_j p_j\cdot p_1 =p_1.$$

Now we move on to proving the right inequality in the claim. We will show $$\sum_j p_j(1-p_j)\leq 2(1-p_1).$$
The LHS is equal to 
\begin{eqnarray*}
p_1(1-p_1) &+& \sum_{j>1}p_j(1-p_j) \leq 1-p_1 + \sum_{j>1} p_j \\ &=& 2(1-p_1)
\end{eqnarray*}

\subsection{Proof of Theorem~\ref{thm:rate-of-convergence}}

Recall that $R(x,\pi)$ is an arbitrary relation on $\cX \times \cE$. For a distribution $\mu$ on $\cX$, we wish to estimate
$$ m^{R}_\mu = \Pr_{X,X' \sim \mu} \left( f(X') = f(X) | R(X', e(X)) \right) .$$

To begin with, let $q(y|\pi)$ denote the probability that a random point $x \sim \mu$ has predicted label $y$ given $R(x,\pi)$:
$$ q(y|\pi) = \Pr_{X \sim \mu}(f(X) = y | R(X,\pi)) .$$
Then we can rewrite our generic faithfulness measure as
$$ m^R_\mu = \E_{X \sim \mu} [q(f(X)|e(X))] .$$

Let $p(\pi)$ be the fraction of points for which explanation $\pi$ is provided, that is, $p(\pi) = \mu(\{x: e(x) = \pi\})$ and let $q(\pi)$ be the fraction for which $R(x,\pi)$ holds: $q(\pi) = \mu(\{x: R(x,\pi)\})$. Note that $p(\pi)$ is a distribution over $\cE$ whereas $q(\pi)\in [0,1]$ and $q(\pi) \geq p(\pi)$.

Given samples $x_1, \ldots, x_n \sim \mu$, and any $y, \pi$, define
\begin{align*}
    N_\pi &= |\{i: R(x_i, \pi)\}| \\
    N_{\pi,y} &= |\{i: R(x_i, \pi), f(x_i) = y\}|
\end{align*}
Our estimator for $m^R_\mu$ is then
$$ \widehat{M} = \frac{1}{n} \sum_{i=1}^n {\bf 1}(N_{e(x_i)} > 1) \, \frac{N_{e(x_i), f(x_i)} - 1}{N_{e(x_i)} - 1} .$$

We start by deriving the expected value of $\widehat{M}$.
\begin{theorem}
For estimator $\widehat{M}$,
$$ \E [ \widehat{M} ] = \E_{X \sim \mu}\left[ (1 - (1-q(e(X)))^{n-1}) q(f(X)|e(X)) \right].$$
\label{thm:estimator-bias}
\end{theorem}
\begin{proof}
Fix any $i \in [n]$. 
The term $\E_{\setminus i}$ denotes expectation over all points other than $i$. We will also use $\E_i$ to denote expectation over point $i$ alone.
Let $y = f(x_i)$ and $\pi = e(x_i)$, and let $k$ be the number of \emph{other} points (that is, $j \neq i$) to which $\pi$ also applies: that is, $k = N_\pi - 1$. Suppose these points are $x_{i_1}, \ldots, x_{i_k}$. If $k > 0$, then $\E_{\setminus i} \left[ \frac{N_{\pi,y}-1}{N_\pi-1} \bigg| N_\pi = k+1 \right]$
is equal to 
$$
\frac{1}{k}\sum_{j=1}^k(\Pr(f(x_{i_j}) = y|R(x_{i_j},\pi))=q(y|\pi).
$$
We then have that $\E[\widehat{M}]$ is equal to 
\begin{align*}
&
\frac{1}{n} \sum_{i=1}^n \E \left[ {\bf 1}(N_{e(x_i)} > 1) \, \frac{N_{e(x_i), f(x_i)} - 1}{N_{e(x_i)} - 1} \right] \\
&=
\frac{1}{n} \sum_{i=1}^n \E_i \left[ {\bf 1}(N_{e(x_i)} > 1) \, \E_{\setminus i} \left[\frac{N_{e(x_i), f(x_i)} - 1}{N_{e(x_i)} - 1} \bigg| N_{e(x_i)} > 1 \right] \right] \\
&=
\frac{1}{n} \sum_{i=1}^n \E_{i} \left[ {\bf 1}(N_{e(x_i)} > 1) q(f(x_i)|e(x_i)) \right] \\
&= 
\E_{X \sim \mu}\left[ (1 - (1-q(e(X)))^{n-1}) q(f(X)|e(X)) \right],
\end{align*}
as claimed.
\end{proof}

Next, we upper-bound the variance of $\widehat{M}$.
\begin{theorem}
$\mbox{\rm var}(\widehat{M}) \leq 4/n$.
\label{thm:estimator-var}
\end{theorem}
\begin{proof}
Suppose $\widehat{M}$ is based on $n$ samples $x_1, \ldots, x_n \sim \mu$. It is not hard to check that changing any one sample, $x_i \to x_i'$, can change $\widehat{M}$ by at most $4/n$. Thus $\widehat{M}$ satisfies a bounded-differences property, whereupon its variance can be bounded by a form of the Efron-Stein inequality (Boucheron, Lugosi, Massart, Cor 3.2). 
\end{proof}

We then sum the bias and variance to get a bound on mean-squared error. From Theorem~\ref{thm:estimator-bias} and the fact that $q(y|\pi) \in [0,1]$, we can bound the bias, $\left| \E[\widehat{M}] - m^R_\mu \right|$ of $\widehat{M}$ by
\begin{align*} 
&
\left| \E_X \left[ (1 - (1-q(e(X)))^{n-1}) q(f(X)|e(X)) \right] - \E_X [q(f(X)|e(X))] \right| \\
&=
\E_X \left[ (1-q(e(X)))^{n-1} q(f(X)|e(X)) \right] \\
&\leq
\E_X \left[ e^{-(n-1)q(e(X))}  \right] \\
&=
\sum_{\pi \in \cE} p(\pi) e^{-(n-1)q(\pi)} .
\end{align*}
Theorem~\ref{thm:rate-of-convergence} then follows by summing the variance and squared bias.

\subsection{Proof of Claim~\ref{clm:unverifiable_explainer}}

\begin{proof}
The claim will hold for any model $f$ which is balanced, i.e., there is the same number of examples labeled $1$ and examples labeled $-1$. Take the distribution over the examples to be the uniform one. 

The explainer $e_1$ returns a different explanation to each example in $\cX$. To define the explainer $e_2$, partition $\cX$ into pairs $(x_1,x_2)$ where $f(x_1)\neq f(x_2)$. Such a partition is possible because $f$ is balanced. Each pair receives the exact explanation $e_2(x_1)=e_2(x_2)$. 

Suppose that for the two explainers $A(x,\pi)\Leftrightarrow e(x)=\pi$.
By definition, the sufficiency and consistency of $e_1$ is one and the sufficiency and consistency of $e_2$ is $0.5$. 

Note that when given a finite sample, since the set of instances $\cX$ is infinite, with zero probability, the example set will contain a pair. Those it is impossible to distinguish if the true explainer is $e_1$ or $e_2$. 
\end{proof}

\newpage
\section{Example where LIME does not have perfect consistency}\label{apx:lime_no_perfect_consistency}
 We show a model and two instances that get different labels but the same explanation by LIME.  In Figure~\ref{fig:xor_model},  we show the XOR model $f^{XOR}$ which is $1$ if the two features have the same sign (in blue). We are using the LIME method to explain two instances $(2.5,0.5)$ and $(-2.5,0.5)$. The model $f^{XOR}$ assigns these two instances different labels. The output of LIME when given instance $(2.5,0.5)$ and $(-2.5,0.5)$ is the same: second feature has the same positive importance of $0.44$ on both instances and the first feature does not have importance, see Figures~\ref{fig:example_1},\ref{fig:example_2}. The reason for such a behavior is that LIME fits a linear classifier around the labeled instance $(x,y)$ where the goal is to predict the class $y$. From the view point of LIME for both of the instances, a linear classifier is fitted for similar training data, see Figure~\ref{fig:example_neighborhoods}.

\begin{figure*}[!htb]
    \centering
    \subfloat[XOR classification]{
         \includegraphics[width=.4\textwidth]{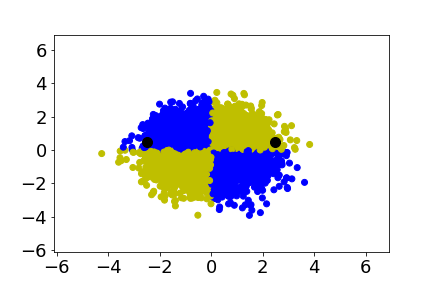}
        \label{fig:xor_model}
    }
    \hfill
     \subfloat[Explanation for $(2.5,0.5)$]{
         \includegraphics[width=.55\textwidth]{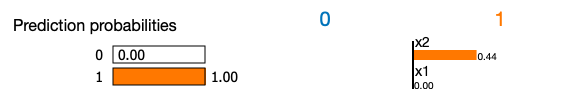}
        \label{fig:example_1}
     }\\
     \subfloat[Neighborhoods]{
        \includegraphics[width=.4\textwidth]{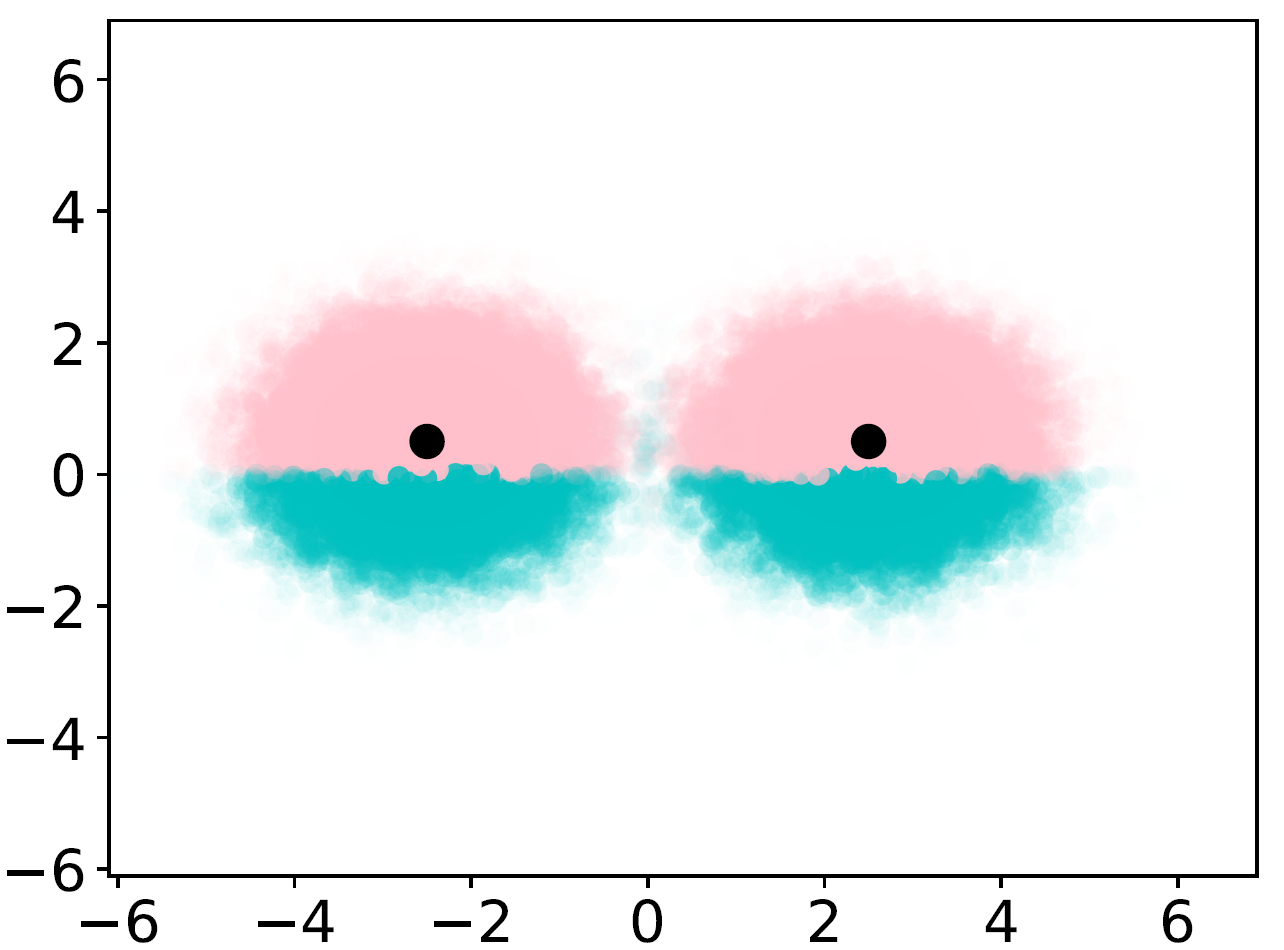}
        \label{fig:example_neighborhoods}
    }
    \hfill
    \subfloat[Explanation for $(-2.5,0.5)$]{
        \includegraphics[width=.55\textwidth]{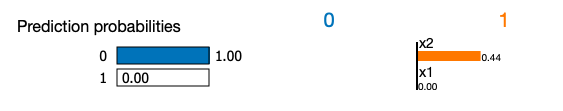}
        \label{fig:example_2}
     }
    \caption{
    (a) XOR model with two instances (in red) with different labels (b,d) LIME provides the same explanation to these instances:  second feature has the same positive importance, $0.44$, on both instances and the first feature does not have importance. (c) During the run of LIME explainer on the two instances, the training data supplied to the linear predictor.
}
    \label{fig:optimal_vs_tree}
\end{figure*}

\newpage
\section{Local estimators}
In this section we explore estimators for the local measures. Namely, Algorithm~\ref{alg:estimate_meaasre} estimates the local consistency and sufficiency measures of explainer $e$ for model $f$ at instance $x$. It uses, as an input, \emph{unlabelled} test data $S$ drawn from distribution $\mu$.
To estimate consistency it returns the fraction of instances with similar label out of all instances with similar explanation. To estimate sufficiency, it returns  the fraction of instances with similar label out of all examples that $e(x)$ applied to.

\begin{algorithm}[!htb]
\begin{algorithmic}
\caption{Estimating local consistency and sufficiency}
\STATE input: model $f$, instance $x$, unlabelled test data $S$
\STATE output: estimate of consistency and sufficiency
\STATE $con_{counter},con_{tot}=0,0$
\STATE $suf_{counter},suf_{tot}=0,0$
\FOR {$x' \in S$}
\IF {$e(x') = e(x)$}
\STATE $con_{counter} += (f(x)==f(x'))$
\STATE $con_{tot}++$
\ENDIF
\IF {$A(x', e(x))$}
\STATE $suf_{counter} += (f(x)==f(x'))$
\STATE $suf_{tot}++$
\ENDIF
\ENDFOR
\STATE \textbf{return} $con_{counter}/con_{tot}, suf_{counter}/suf_{counter}$
\label{alg:estimate_meaasre}
\end{algorithmic}
\end{algorithm}

If we have a random sample $S$ from $C_\pi$ then, by Hoeffding's inequality, it is enough to take sample size $|S|=O(1/\epsilon^2)$ to approximate the consistency measure up to an additive error of $\epsilon$ with constant probability. This is summarized in the following corollary. 

\begin{corollary}
Fix $\epsilon\in(0,1)$ and an instance $x$. Given a sample of size $O(1/\epsilon^2)$ from $C_{e(x)}$, one can estimate $m^c(x)$ up to an additive error $\epsilon$ with probability $0.9$. 
\end{corollary}

The difficulty with the above corollary is the assumption that one can obtain enough samples from $C_{e(x)}$. This assumption is sometimes unrealistic. To get an instance from $C_{e(x)}$, one can use \emph{rejection sampling}. Where instances are received from arbitrary distribution, but then reject any instance that is not in $C_{e(x)}$. Although this is a reasonable technique, it might take a long time till an instance from $C_{e(x)}$ is received.

\newpage
\section{More experimental details}

\subsection{Datasets}
\label{sec:appendix_datasets}

Datasets in the empirical evaluation are depicted in Table \ref{tab:datasets}.

\begin{table}[!htb]
    \centering
    \caption{Datasets properties}
    \label{tab:datasets}
    \begin{tabular}{l l l l}
    \toprule
    Dataset & \# of classes & $n$ & $d$ \\
    \midrule
    Heart \cite{DETRANO1989304} & 2 & 303 & 13 \\
    Chess \cite{Dua:2019} & 17 & 28,056 & 6 \\
    Avila \cite{de2018reliable} & 12 & 20,867 & 10 \\
    Bank marketing \cite{moro2014data} & 2 & 45,211 & 16 \\
    Adult \cite{kohavi1996scaling} & 2 & 48,842 & 14 \\
    Covtype \cite{blackard1999comparative}& 7 & 581,012 & 54 \\
    rt-polaritydata \cite{Pang+Lee:05a} & 2& 10,433 & 15,888\\
    \bottomrule
    \end{tabular}
\end{table}

\subsection{Model training}
\label{sec:appendix_traning}

In sections~\ref{subsec:exp_common_explainers} and~\ref{subsec:exp_data_dependent} we have explained gradient boosted trees models trained over 6 datasets.
For each dataset, 66\% of it was used for model training and cross-validation.
Hyper-parameters were selected based on best mean accuracy over 3 cross-validation executions. 
The considered hyper-parameters are all combinations of the following:
\begin{itemize}
    \item \texttt{learning\_rate}: $2^{-5}, 2^{-4}, \ldots, 2^2.$
    \item \texttt{n\_estimators}: $50, 100, 150, 200, 250, 300.$
    \item \texttt{max\_depth}: $3, 4, 5, 6, 7.$
\end{itemize}

The selected hyper-parameters and test accuracy is presented in Table~\ref{tab:training_hyper_parm}.

\begin{table}[!htb]
    \centering
    \caption{Gradient boosted trees hyper-parameters and accuracy}
    \label{tab:training_hyper_parm}
    \begin{tabular}{l l l l l}
    \toprule
    Dataset & \texttt{learning\_rate} & \texttt{n\_estimators} & \texttt{max\_depth} & Test accuracy \\
    \midrule
    Heart \cite{DETRANO1989304} & 0.0625 & 250 & 3 & 0.8\\
    Chess \cite{Dua:2019} & 0.0625 & 300 & 7 & 0.9 \\
    Avila \cite{de2018reliable} & 0.125 & 300 & 5 & 0.99\\
    Bank marketing \cite{moro2014data} & 0.0625 & 250 & 5 & 0.91 \\
    Adult \cite{kohavi1996scaling} &  0.25 & 50 & 5 & 0.87\\
    Covtype \cite{blackard1999comparative} &  0.125 & 300 & 7 & 0.94 \\
    \bottomrule
    \end{tabular}
\end{table}

\newpage
\subsection{Sample complexity experiment}\label{sec:appendix_sample_complexity}

In Section~\ref{subsec:canonical_properties} Figure~\ref{fig:main_adult_decision_tree} we have studied the sample complexity of decision tree model and explainer over \texttt{Adult} dataset.
Figure~\ref{fig:decision_tree} depict the same concept over additional datasets.

\begin{figure*}[!htb]
    \centering
    \subfloat[Chess]{
         \includegraphics[width=.3\linewidth]{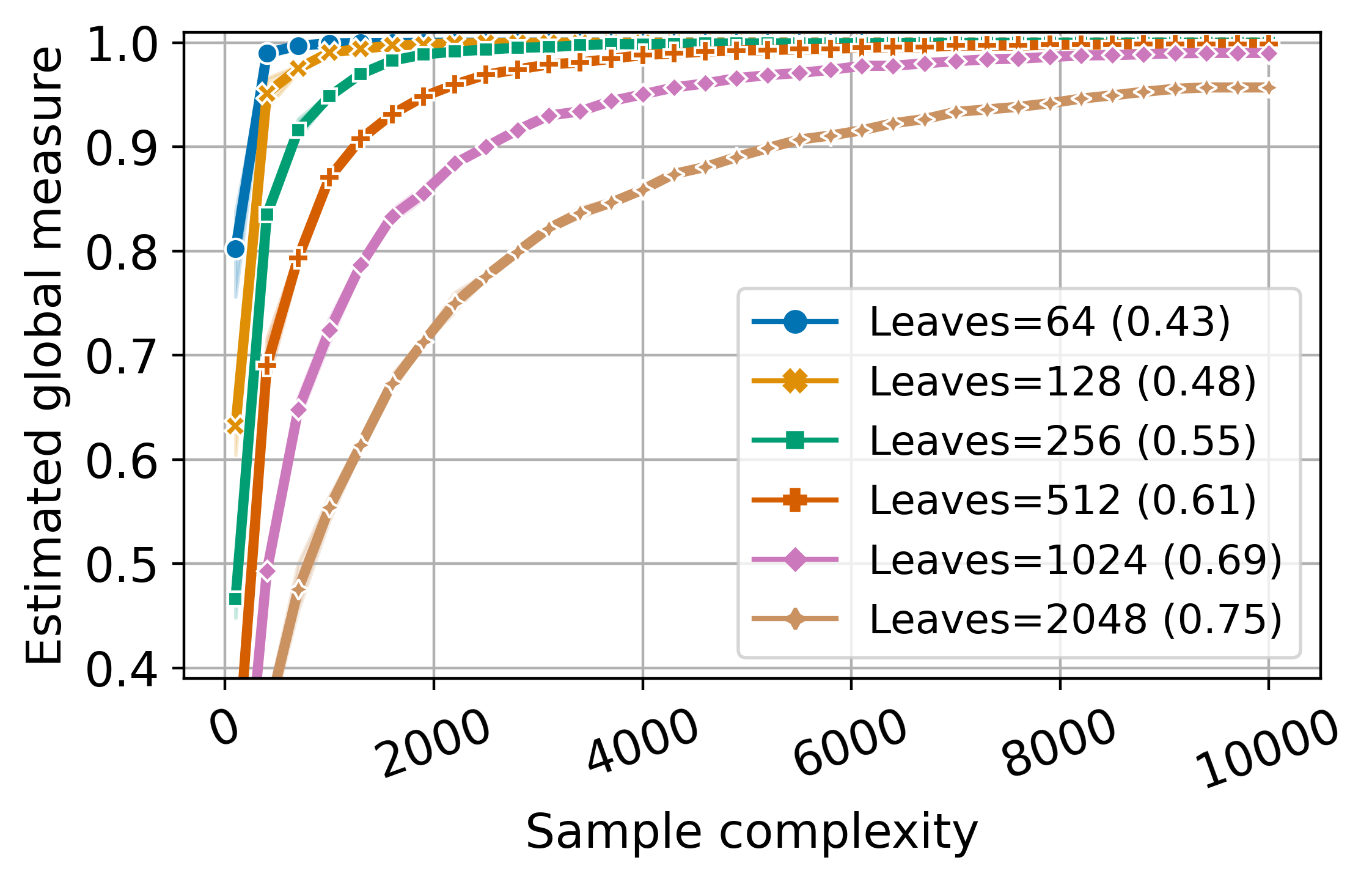}
    \label{fig:decision_tree_chess}}
     \subfloat[Avila]{
         \includegraphics[width=.3\linewidth]{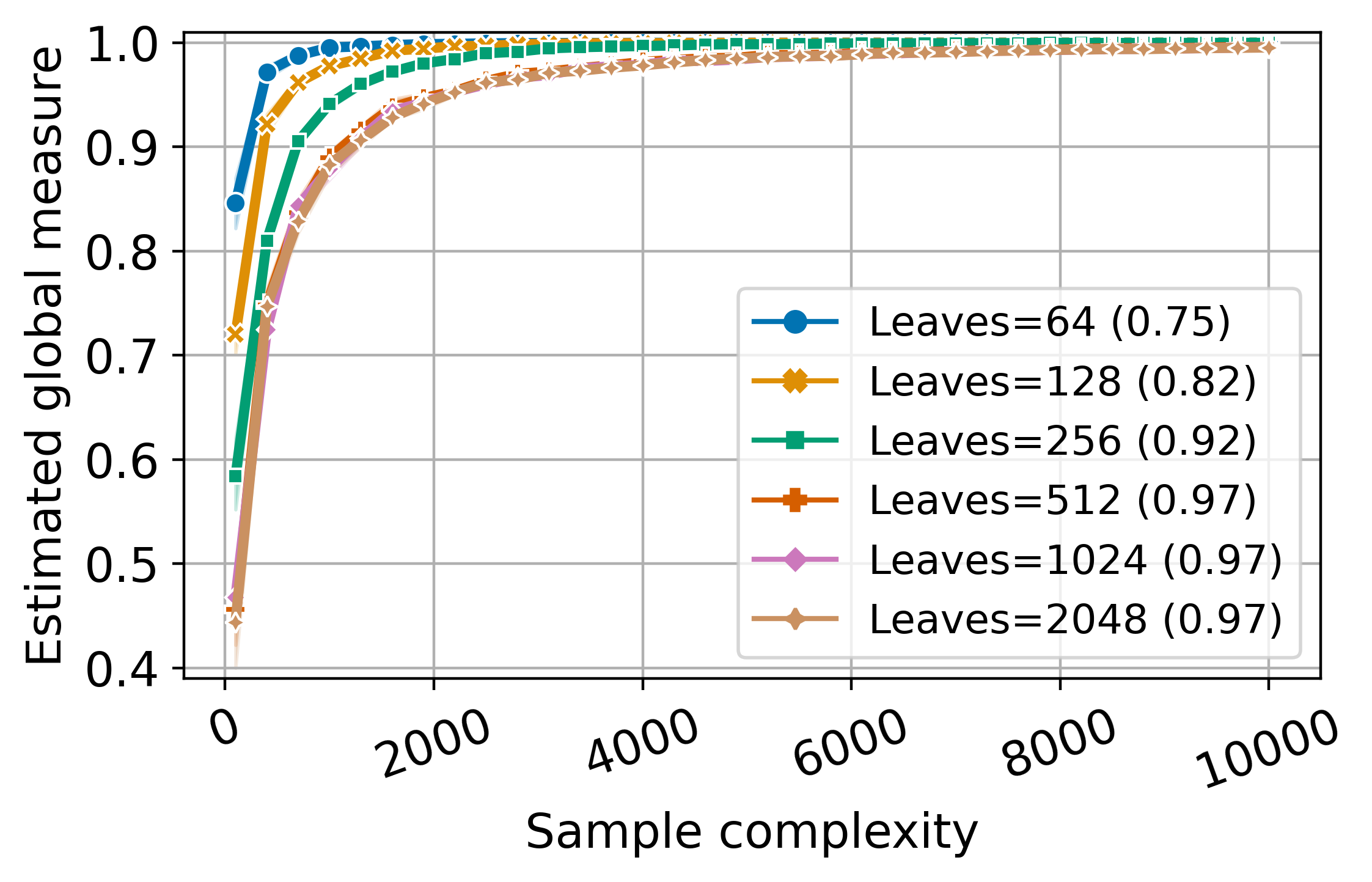}
    \label{fig:decision_tree_avila}}
    \subfloat[Bank marketing]{
         \includegraphics[width=.3\linewidth]{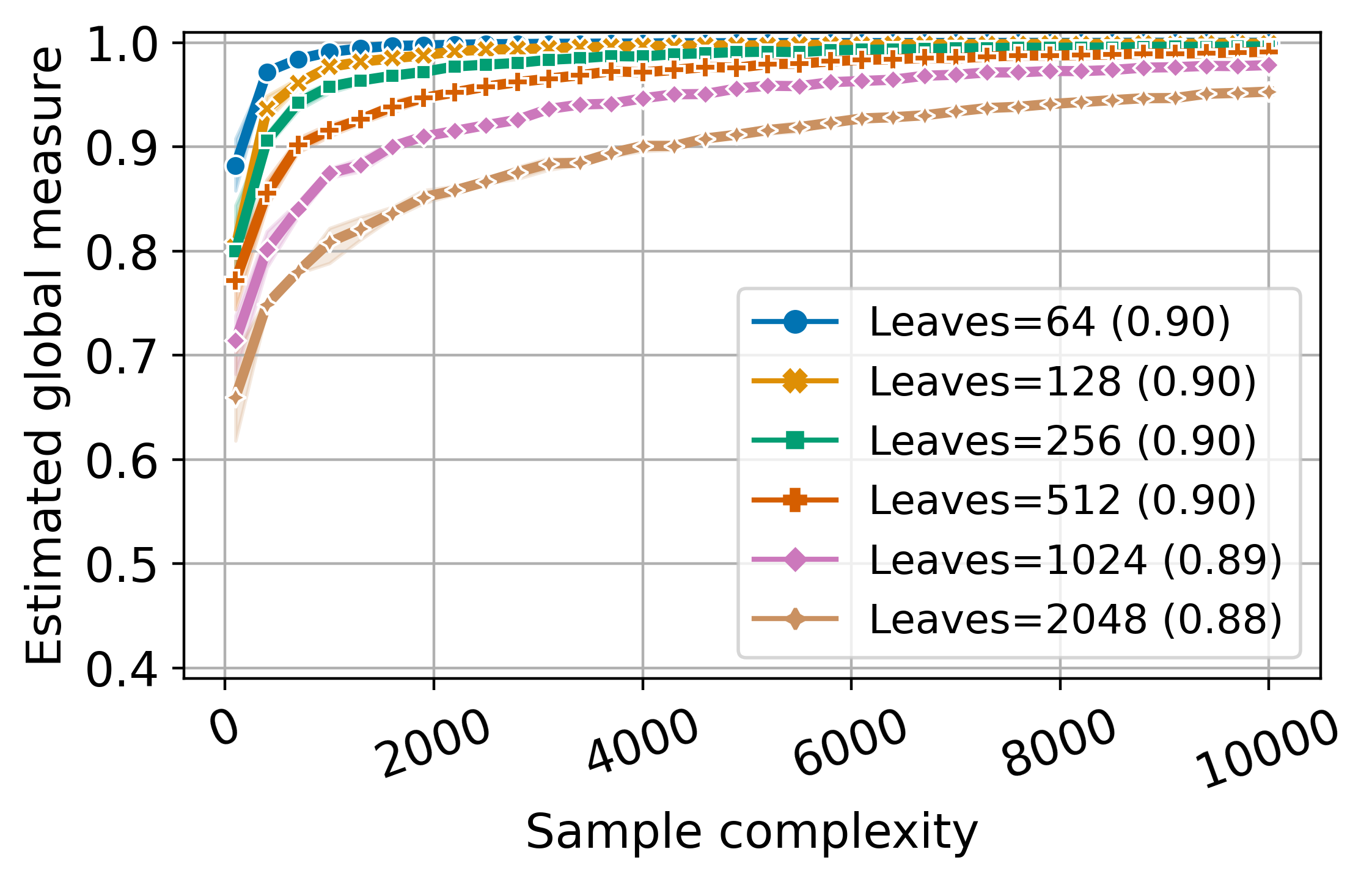}
    \label{fig:decision_tree_bank}}
    \\
    \subfloat[Adult]{
         \includegraphics[width=.3\linewidth]{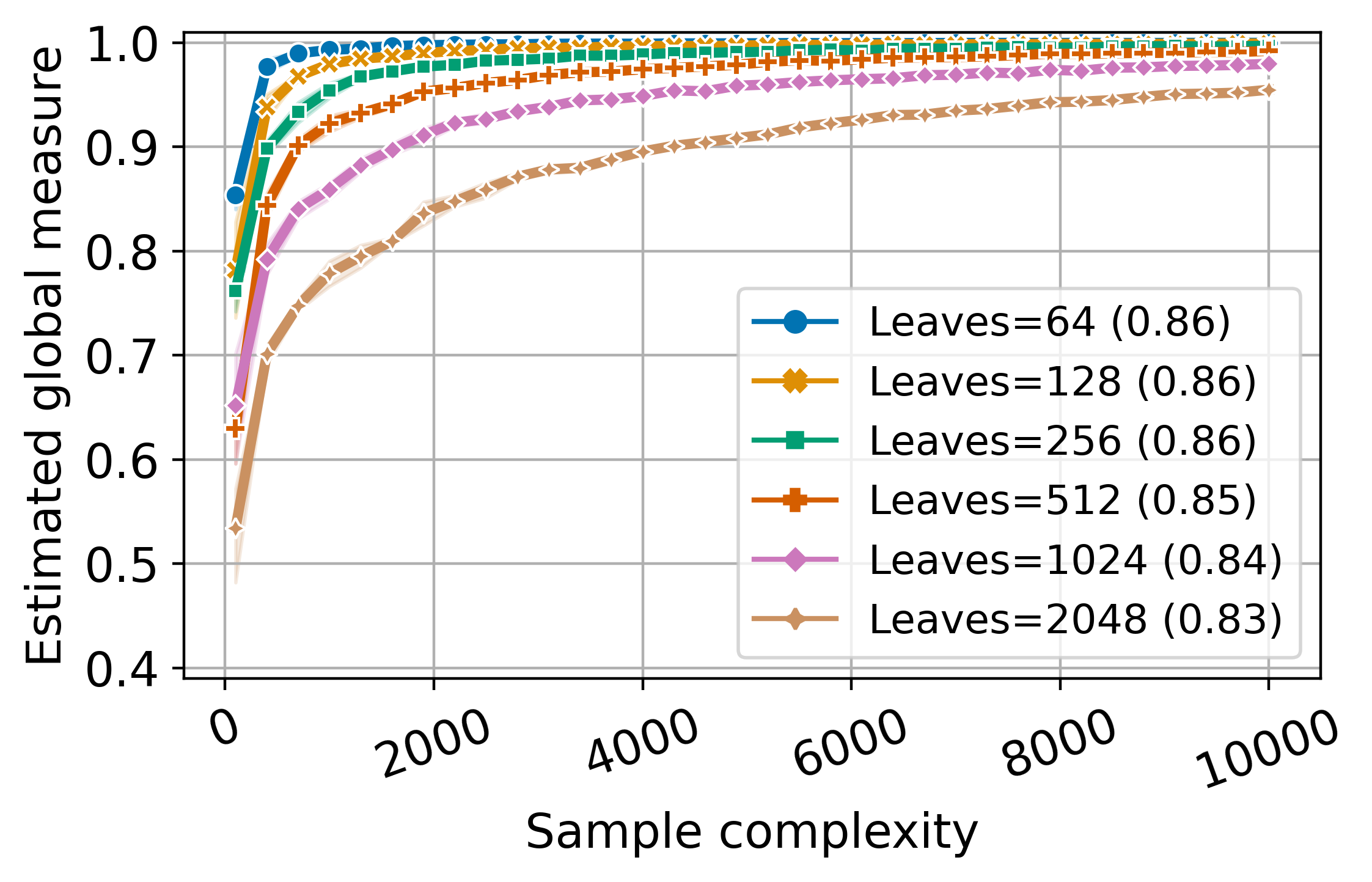}
    \label{fig:decision_tree_adult}}
    \subfloat[Covtype]{
         \includegraphics[width=.3\linewidth]{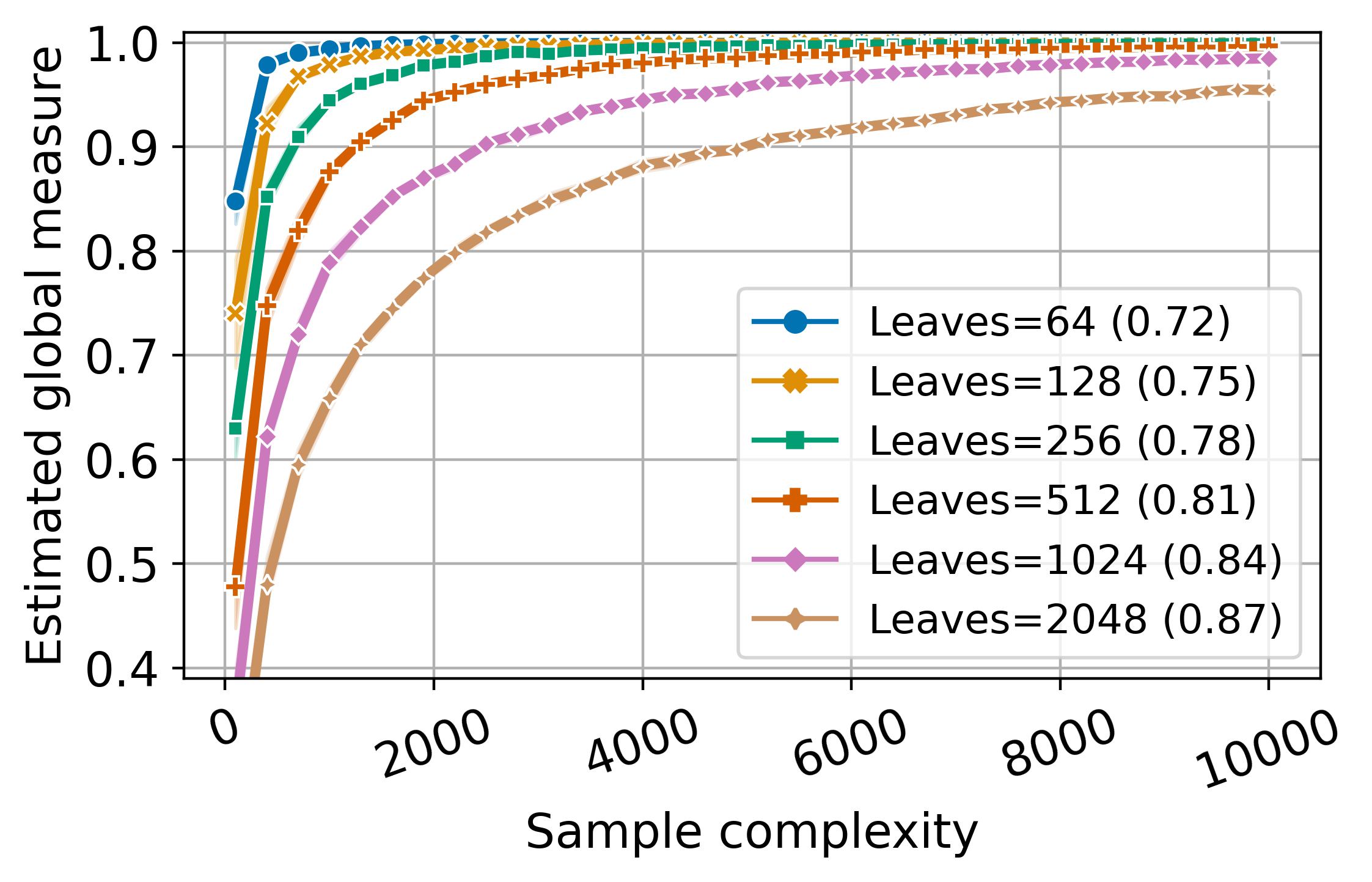}
    \label{fig:decision_tree_covtype}}    
    \caption{Estimated global consistency \& sufficiency of decision trees with different sizes on 5 datasets.
    As sample complexity grows the estimation is getting closer to the ground truth measures (1.0). Larger trees has more leaves, which implies a larger explanations domain. Decision tree accuracy over the full test set is reported in the legend parenthesis.
    The displayed results are the mean of 5 executions with confidence interval of 95\%.}
    \label{fig:decision_tree}
\end{figure*}

\begin{figure*}[!htb]
    \centering
    \subfloat[Chess]{
         \includegraphics[width=.3\linewidth]{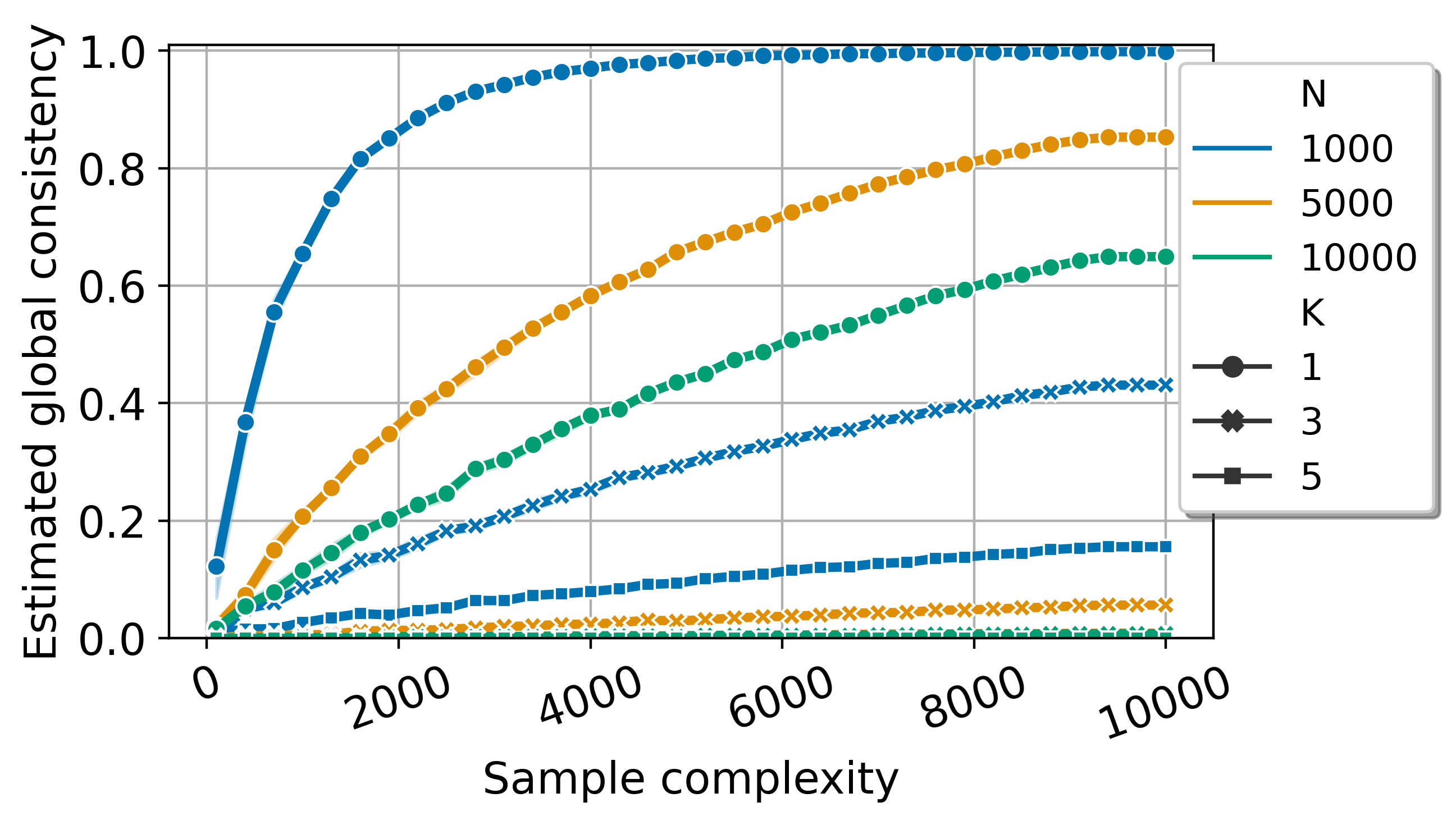}
    \label{fig:knn_chess}}
     \subfloat[Avila]{
         \includegraphics[width=.3\linewidth]{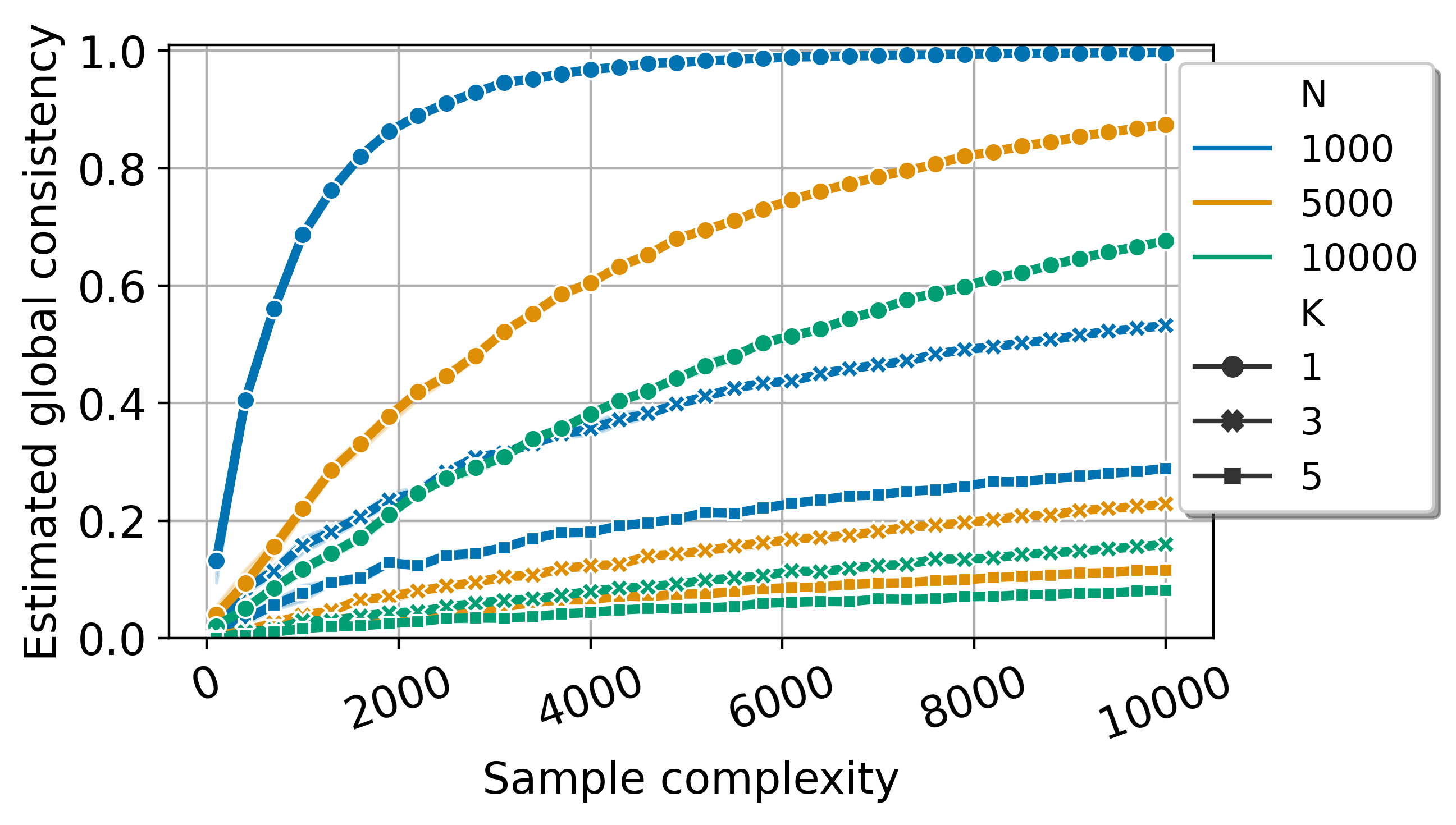}
    \label{fig:knn_avila}}
    \subfloat[Bank marketing]{
         \includegraphics[width=.3\linewidth]{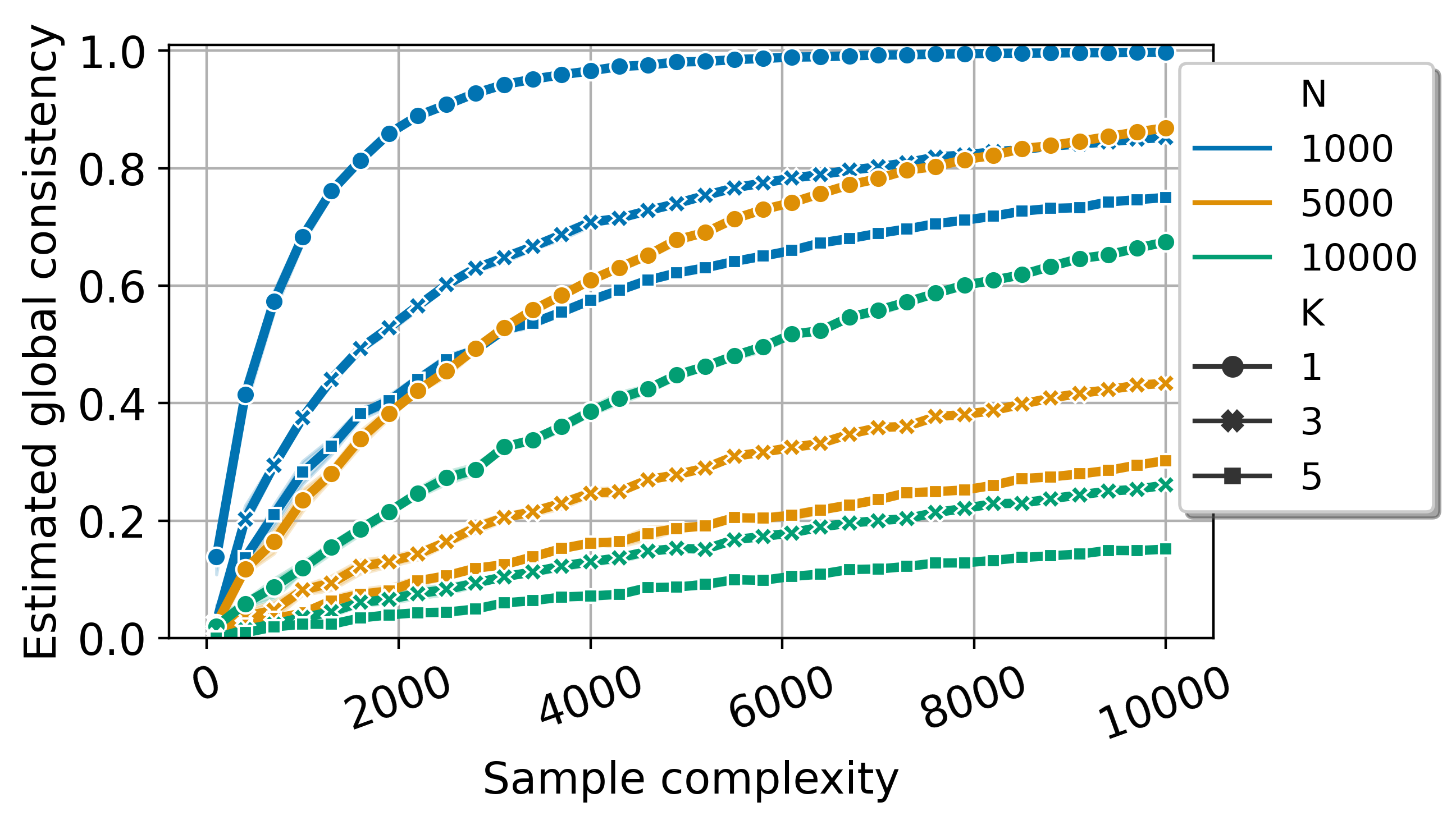}
    \label{fig:knn_bank}}
    \\
    \subfloat[Adult]{
         \includegraphics[width=.3\linewidth]{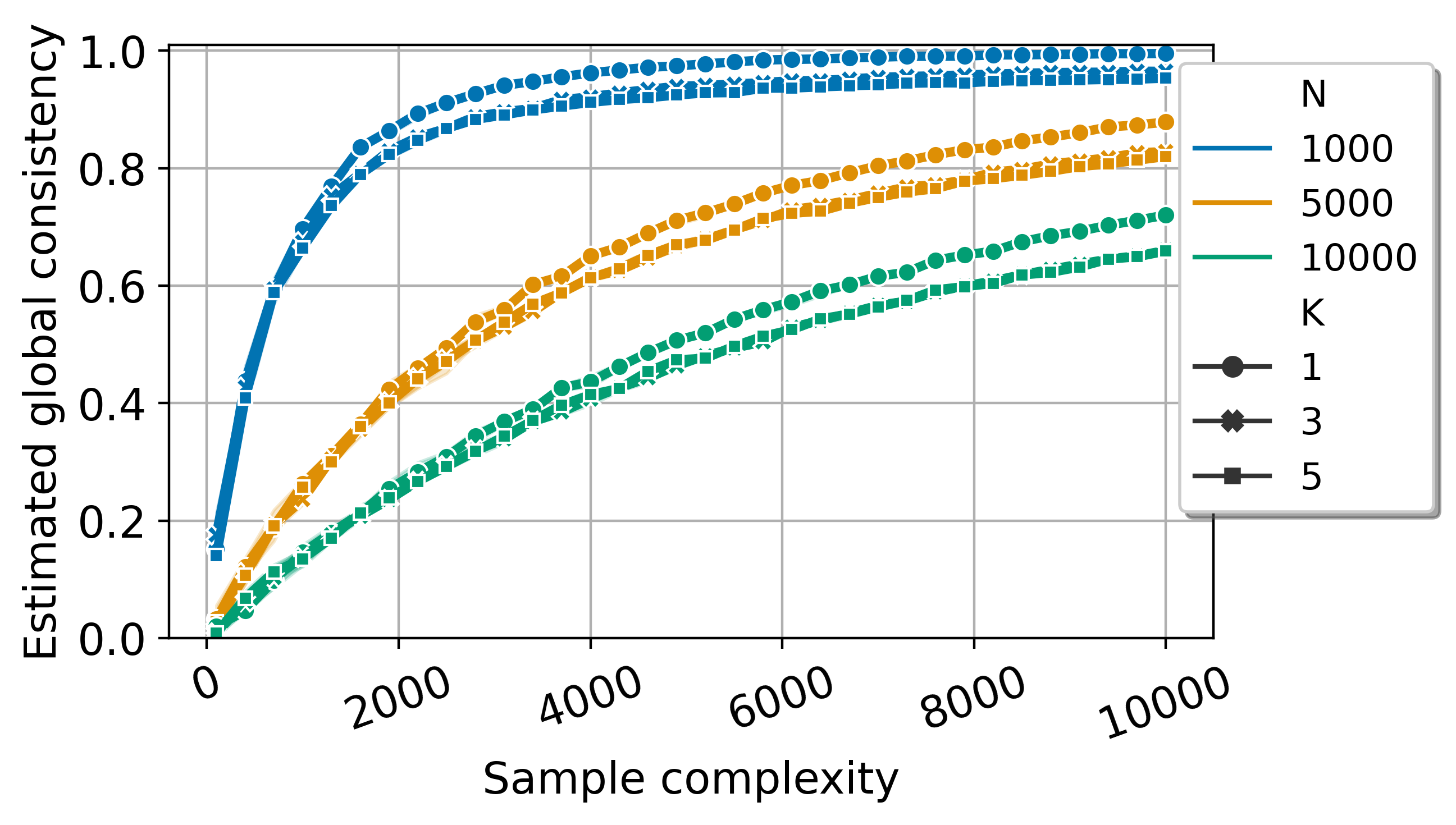}
    \label{fig:knn_adult}}
    \subfloat[Covtype]{
         \includegraphics[width=.3\linewidth]{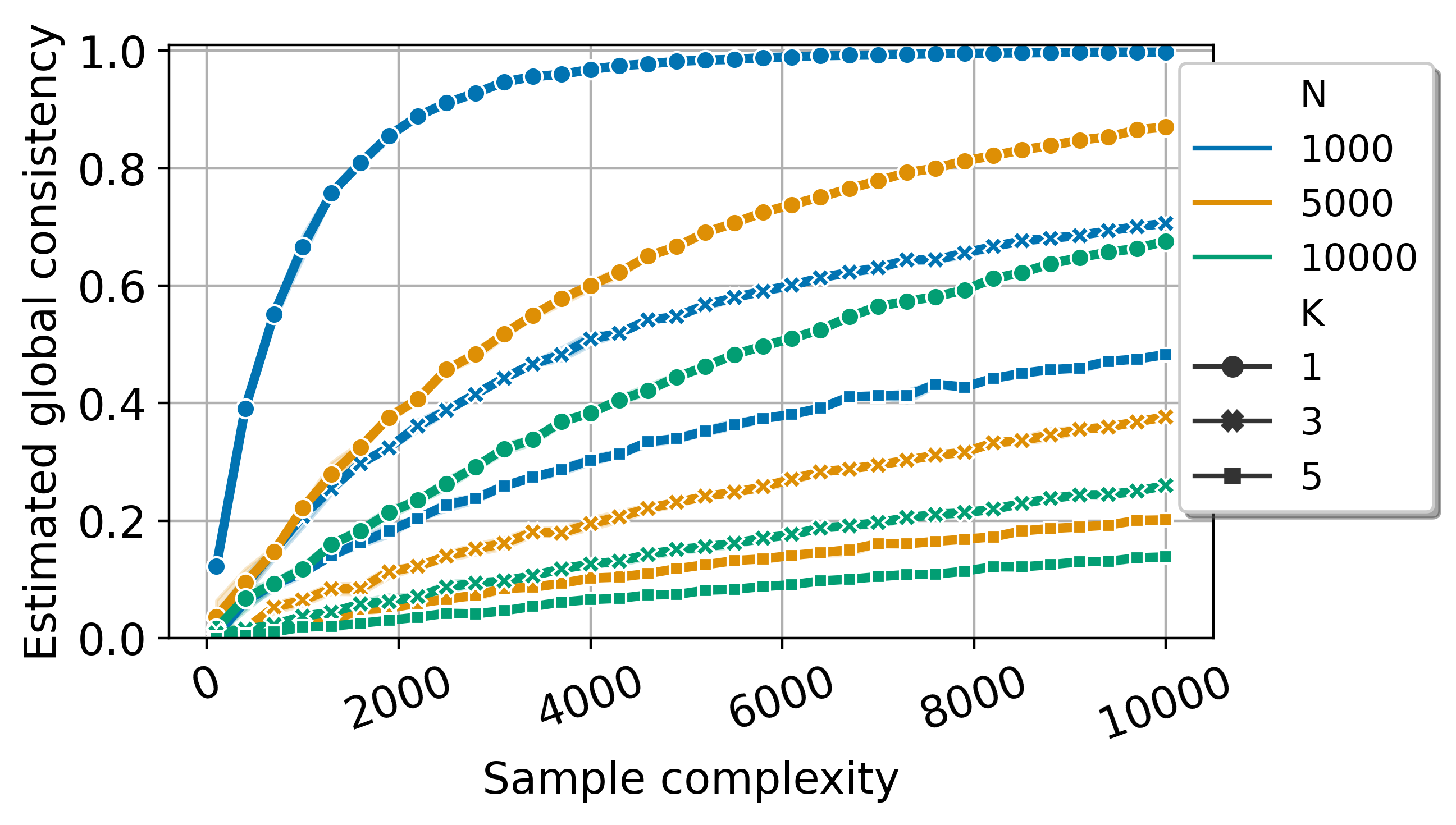}
    \label{fig:knn_covtype}}    
    \caption{Estimated global consistency of $k$ nearest neighbors with different sizes on 5 datasets.
    As sample complexity grows the estimation is getting closer to the ground truth measures (1.0). Larger $k$ and larger training set size ,$N$, implies a larger explanations domain. Accuracy over the full test set is reported in the legend parenthesis.
    The displayed results are the mean of 5 executions with confidence interval of 95\%.}
    \label{fig:knn}
\end{figure*}

Similarly, Figure~\ref{fig:knn} depict the sample complexity required for the evaluation of $k$ nearest neighbors model and explainer. As the explainer and model are the same, the explainer consistency is 1 by definition. Figure~\ref{fig:knn} shows that as $k$ or $N$ (number of training examples) increases, the explanations space grows, and as a result, more samples are required to accurately estimate the explainer consistency.

\subsection{Anchors dependency on precision \threshold{} parameter}\label{sec:appendix_anchors_precision}

Figure~\ref{fig:anchor} depict how the explainer's parameters affect global measures.
The Figure displays the measures of Anchors explainer, applied over gradient boosted trees trained over six datasets, as a function of the precision \threshold{} parameter. 
Similarly to the findings obtained in Figure~\ref{fig:adult_anchor_estimates}, one may see that as the precision increases, the sufficiency and uniqueness increases, while the estimated consistency decreases.

\begin{figure*}[!htb]
    \centering
    \subfloat[Heart]{
         \includegraphics[width=.3\linewidth]{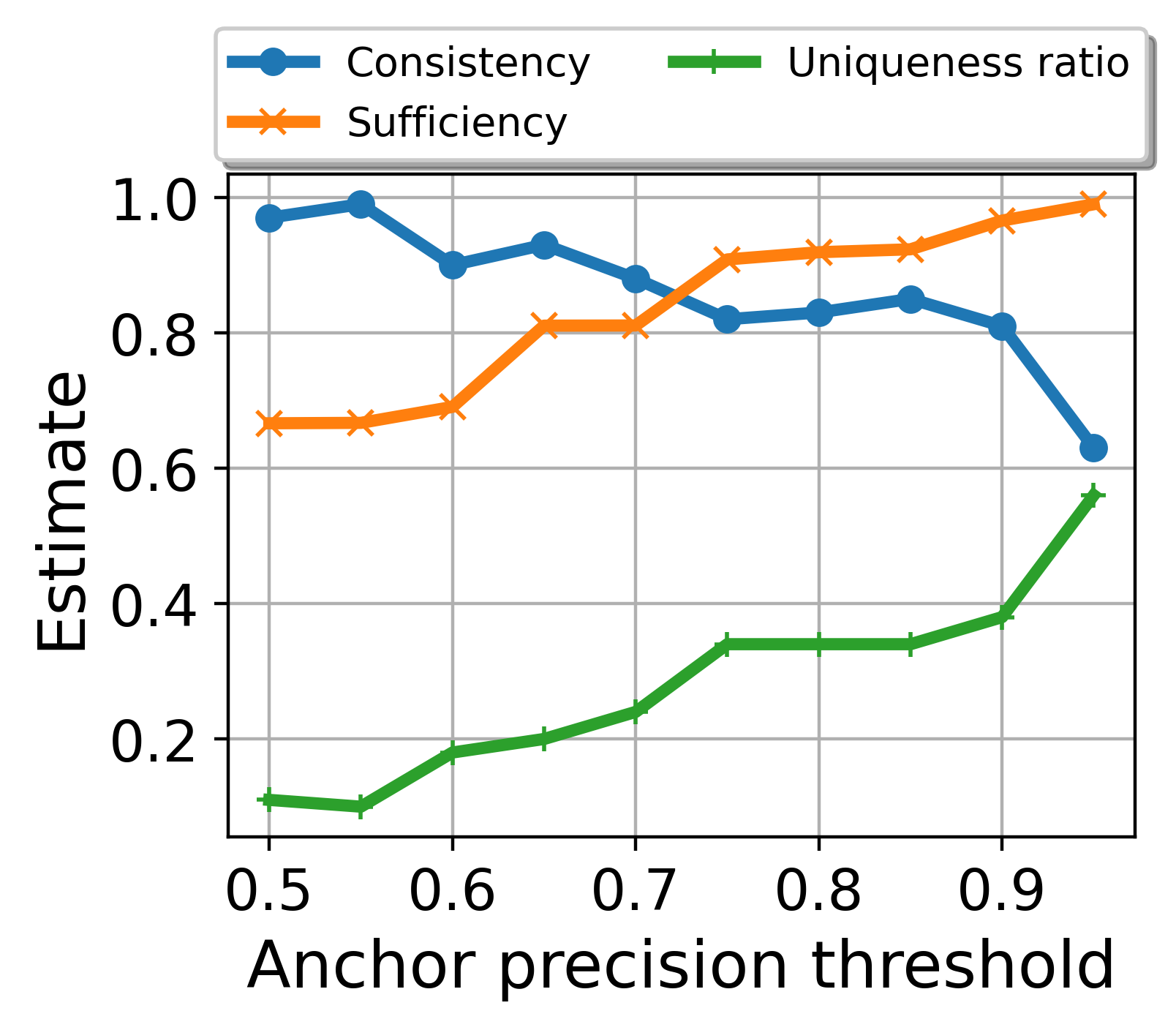}
    \label{fig:anchor_heart}}
    \subfloat[Chess]{
         \includegraphics[width=.3\linewidth]{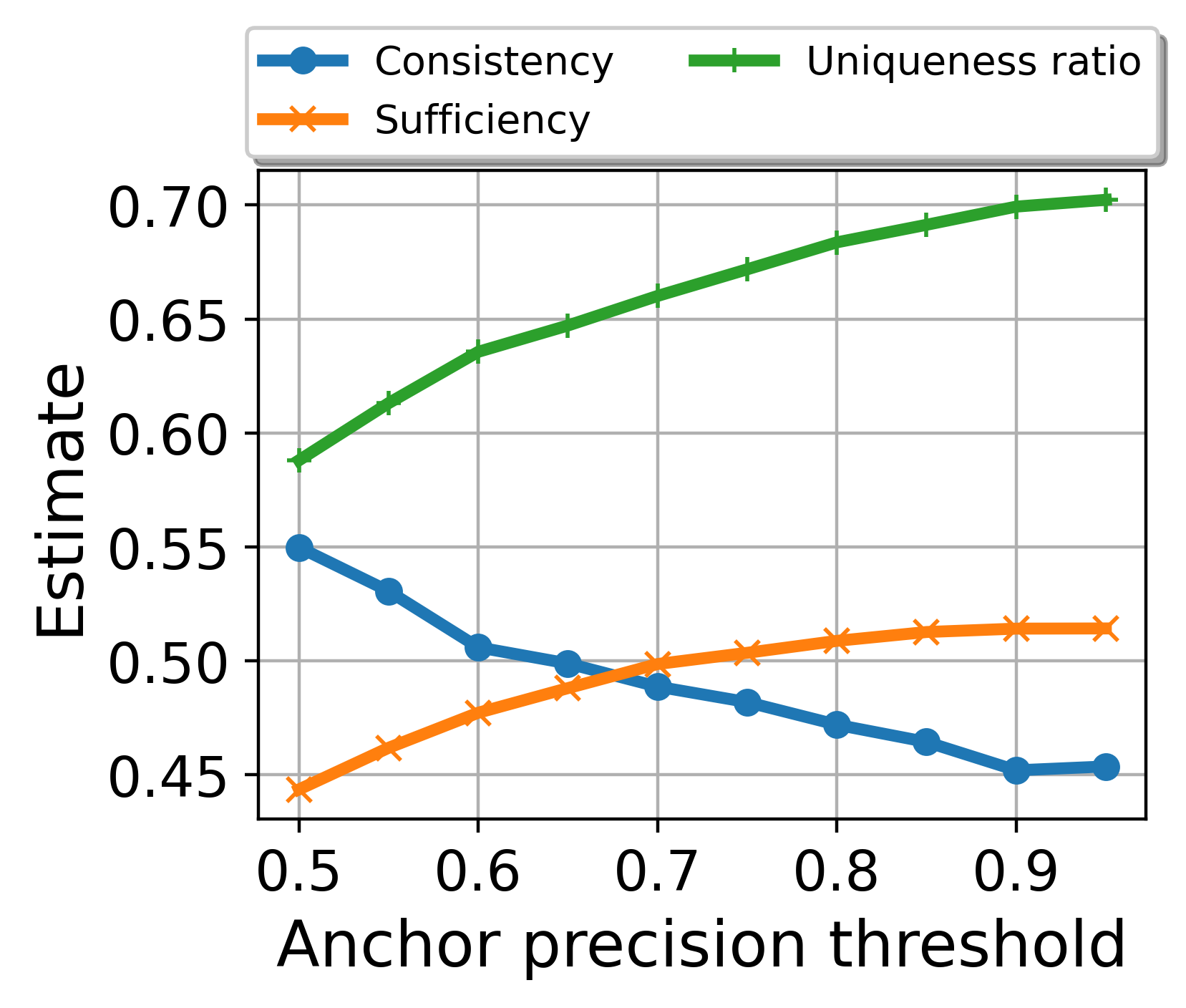}
    \label{fig:anchor_chess}}
     \subfloat[Avila]{
         \includegraphics[width=.3\linewidth]{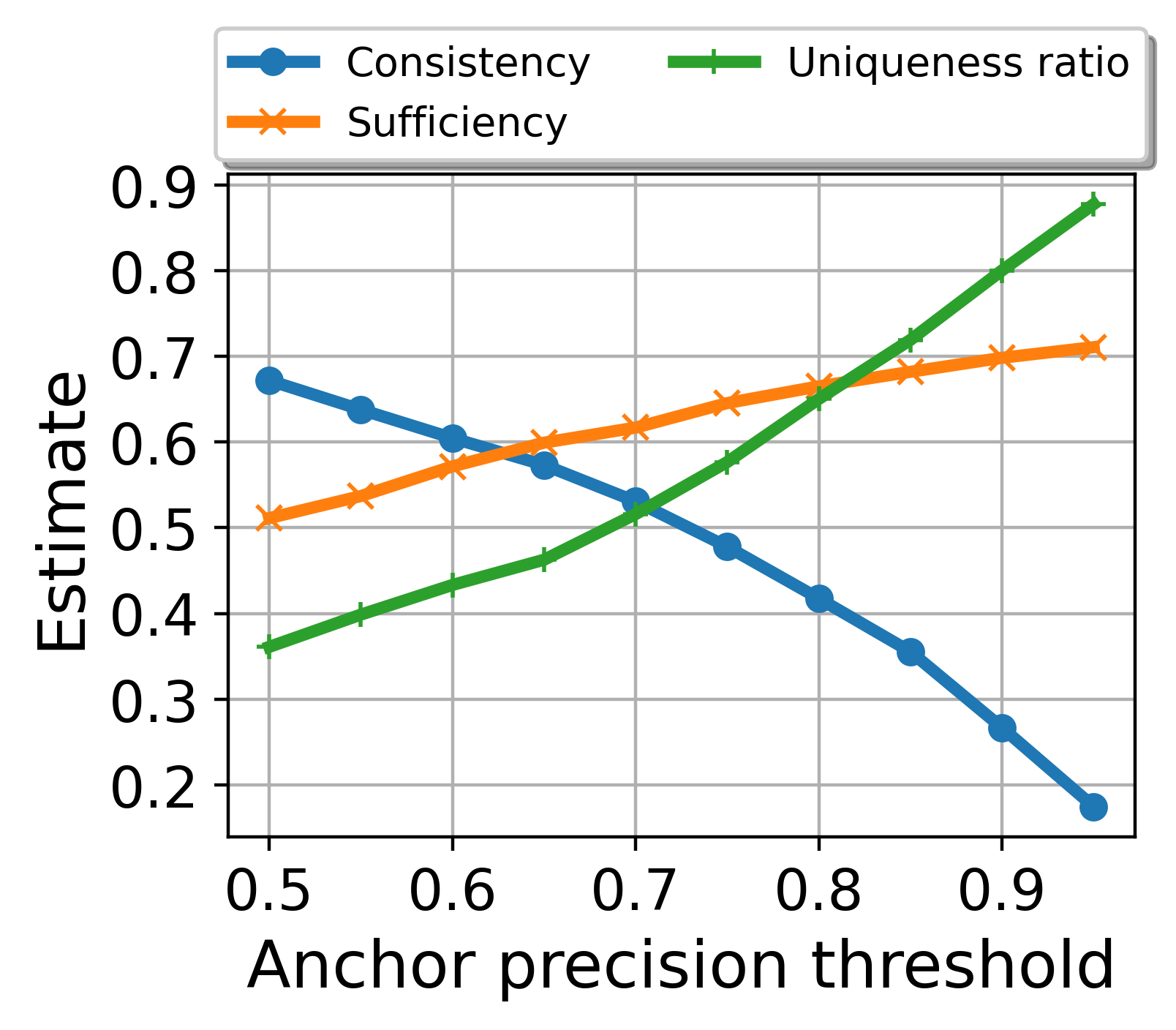}
    \label{fig:anchor_avila}}\\
    \subfloat[Bank marketing]{
         \includegraphics[width=.3\linewidth]{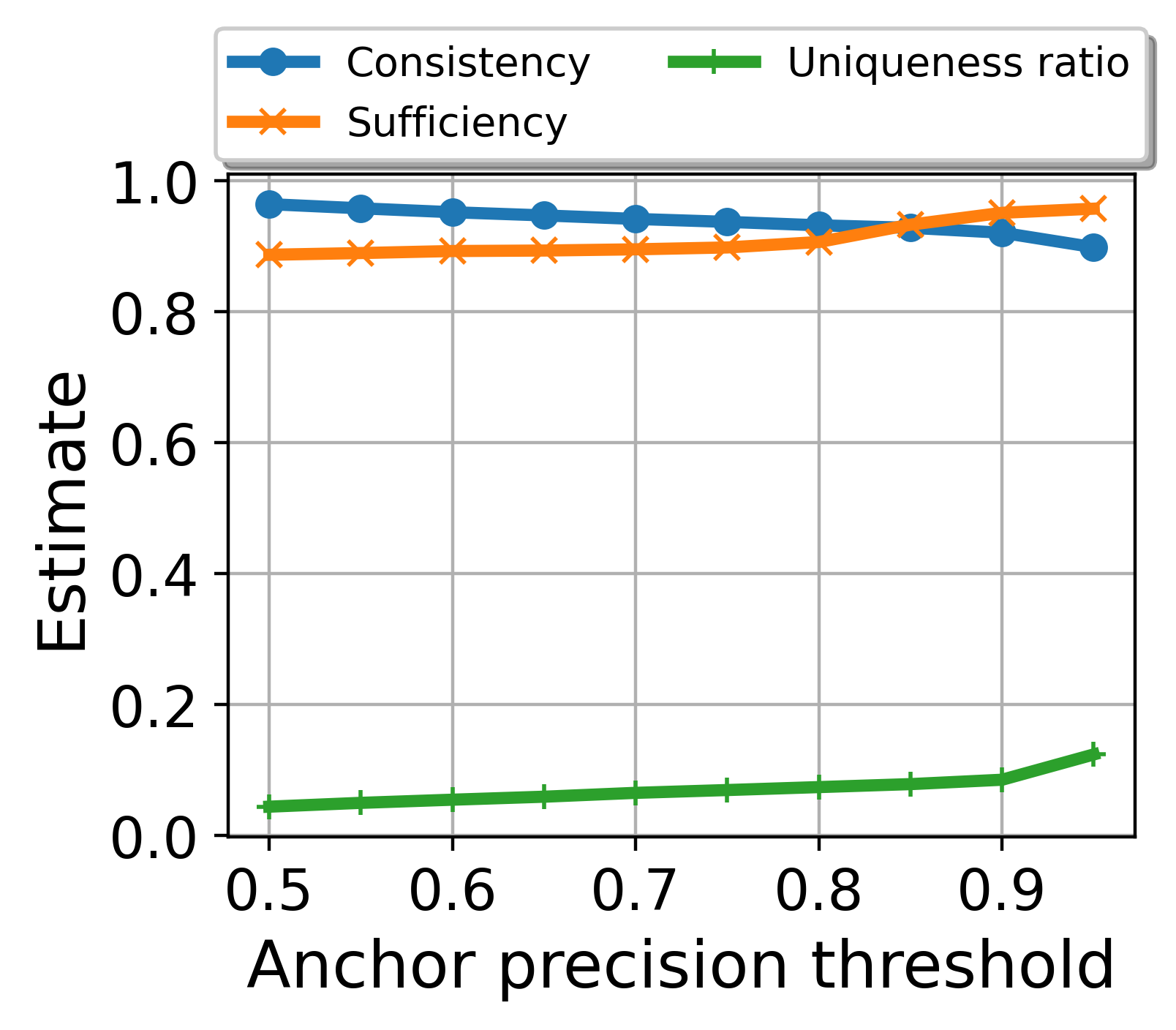}
    \label{fig:anchor_bank}}
    \subfloat[Adult]{
         \includegraphics[width=.3\linewidth]{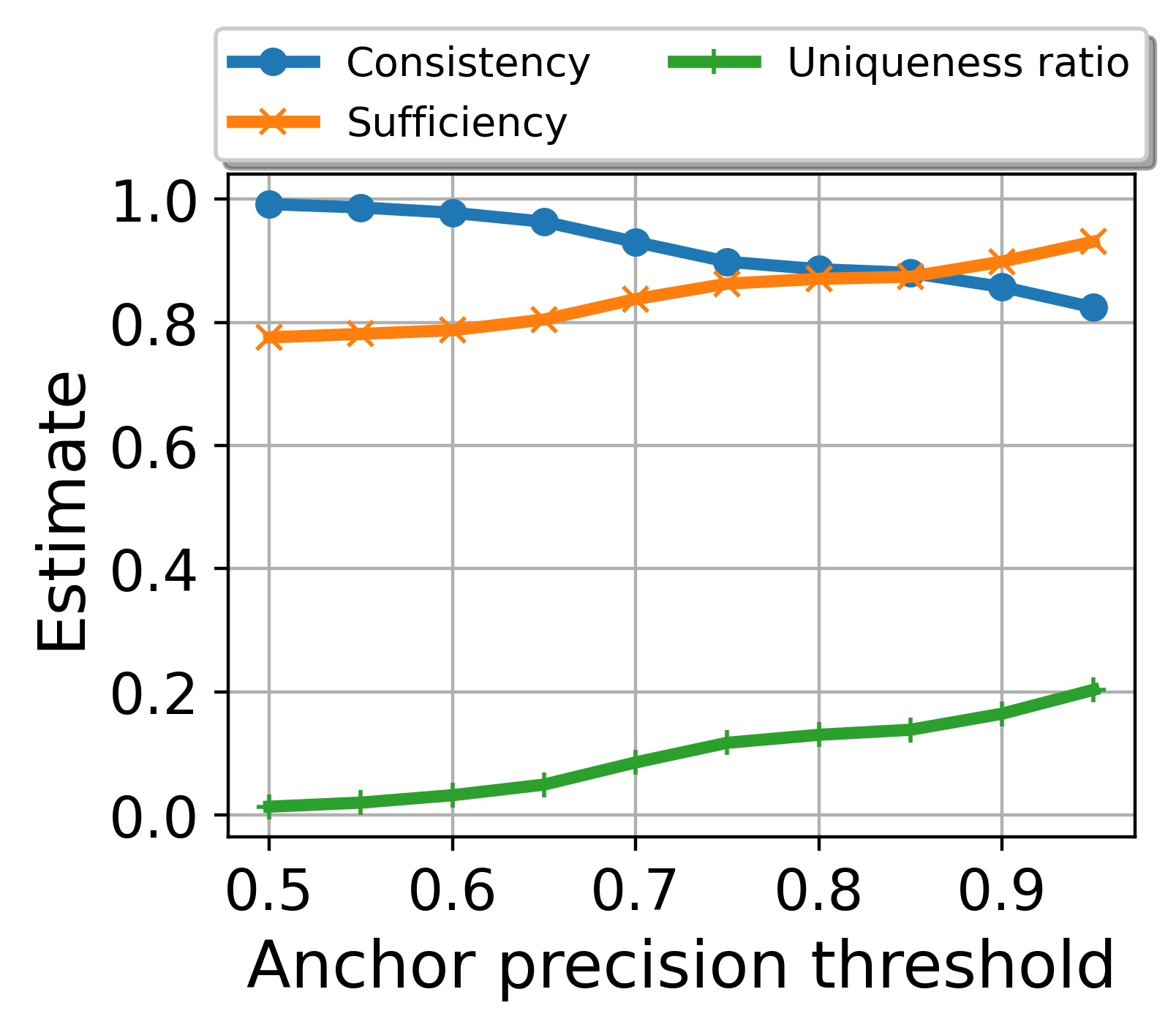}
    \label{fig:anchor_adult}}
    \subfloat[Covtype]{
        \includegraphics[width=.3\linewidth]{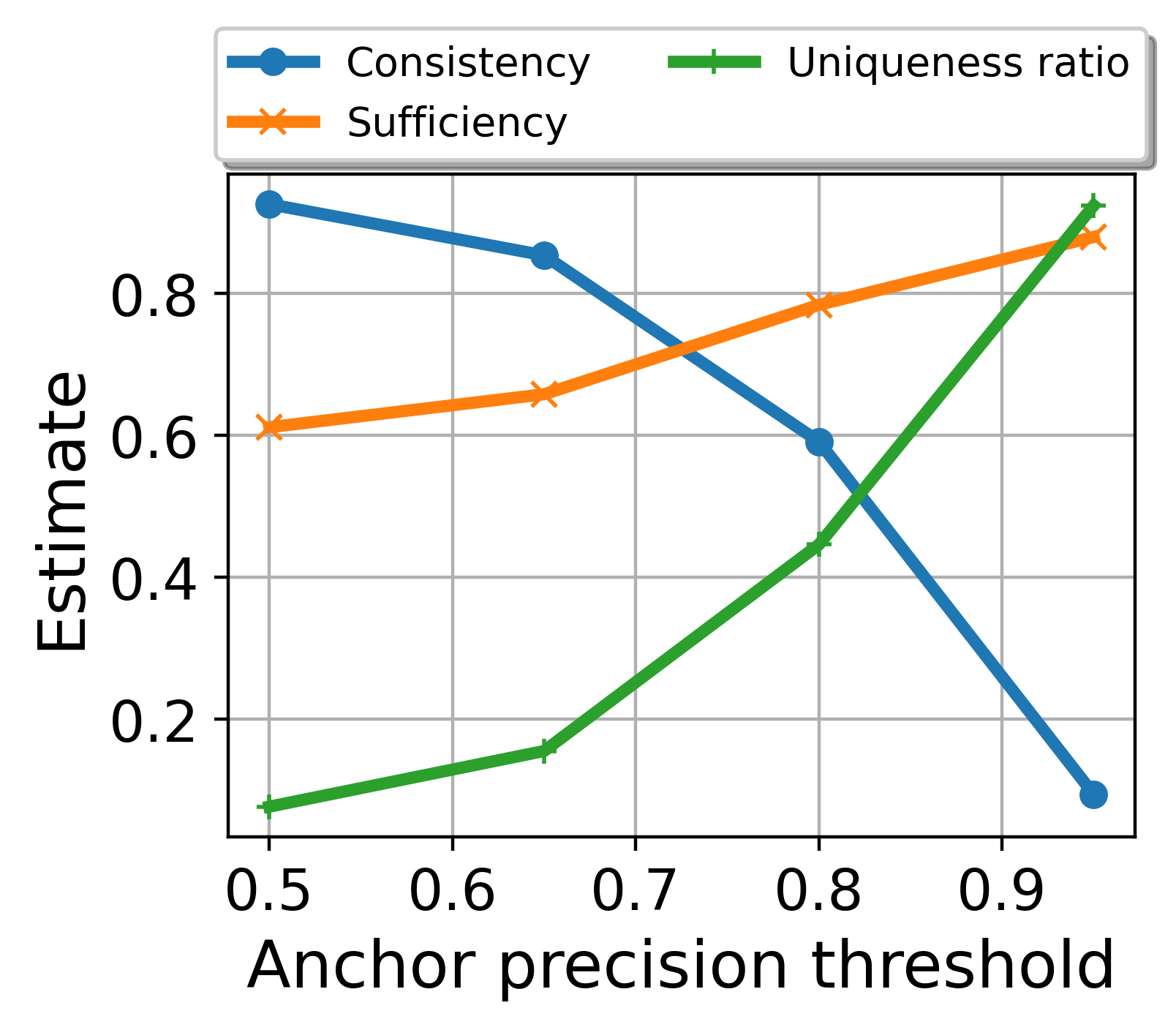}
        \label{fig:anchor_covtype}
    }
    \caption{Estimated global consistency and sufficiency and the number of unique explanations of the Anchors explainer over gradient boosted trees model for 6 dataset as a function of precision \threshold{} parameter.
    As the required precision grows the number of unique explanations (green) grows as well as the estimated sufficiency.
    }
    \label{fig:anchor}
\end{figure*}

\subsection{Explainers discretization}\label{sec:appendix_discretization}

Next we discuss several discretizations we have evaluated. 

\paragraph*{Feature importance}
Recall that for an explanation of type \emph{feature importance}, given an instance $x \in \mathbb{R}^d$ it returns a vector $\phi\in \mathbb{R}^d$ with $\phi_i$ the importance of the $i$-th feature.
For feature importance explainers, i.e. SHAP and LIME we compared the following discretization methods.

\begin{itemize}
    \item \emph{Original}: return $\phi$ as is.
    \item \emph{2-FP}: discretize $\phi$ to have 2 floating-points representation, i.e., return $\phi' \in \mathbb{R}^d$, such that $\phi'_i = \frac{\lfloor 100 \cdot \phi_i \rfloor}{100}$.
    \item \emph{1-FP}: discretize $\phi$ to have a single floating-point representation, i.e., return $\phi' \in \mathbb{R}^d$, such that $\phi'_i = \frac{\lfloor 10 \cdot \phi_i \rfloor}{10}$.
    \item \emph{Sign}: return $\phi' \in \{-1,1\}^d$ such that $\phi'_i = \mathtt{sign}(\phi_i)$.
    \item \emph{Rank}: return $\phi' = \mathtt{argsort}(\phi)$.
    \item \emph{Sign-of-top-5}: let $\phi^+ \in \mathbb{R}^d$ be the vector of absolute values of $\phi$, i.e. $\phi^+_i = \lvert \phi_i \rvert$, and let $\phi^{+,R} = \mathtt{argsort}(\phi^+)$, i.e. $\phi^{+,R}$ is the rank of $\phi$ absolute values. Sign-of-top-5 return $\phi' \in \{-1, 0 ,1\}^d$ such that $\phi'_i = \begin{cases} \mathtt{sign}(\phi_i) & \phi^{+,R}_i > d - 5\\
    0 & else\end{cases}$.
\end{itemize}

Tables~\ref{tab:shap_disc_full} and~\ref{tab:lime_disc_full} depict the consistency and uniqueness ratio of the above discretizations for SHAP and LIME respectively.

\begin{table*}[!htb]
    \centering
    \scriptsize
    \caption{SHAP consistency scores and uniqueness ratio for various discretizations, averaged over 5 executions (std is lower than 0.01 in all cases).}
    \label{tab:shap_disc_full}
    \begin{tabular}{l | l l | l l | l l | l l | l l | l l }
        \toprule
         \multirow{2}{*}{Dataset} &
         \multicolumn{2}{c |}{Original} &
         \multicolumn{2}{c |}{2-FP} &
         \multicolumn{2}{c |}{1-FP} & 
         \multicolumn{2}{c |}{Sign} &
         \multicolumn{2}{c |}{Rank} &
         \multicolumn{2}{c}{Sign-of-top-5} \\
         & 
         Cons. & Uniq. &
         Cons. & Uniq. &
         Cons. & Uniq. &
         Cons. & Uniq. &
         Cons. & Uniq. &
         Cons. & Uniq. \\
         \midrule
         Heart & 0.0 & 1.0 & 0.0 & 1.0 & \textbf{0.48} & 0.70 & 0.02 & 0.99 & 0.02 & 0.99 & 0.39 & 0.75 \\
         Chess & 0.0 & 1.0 & 0.0 & 1.0 & 0.0 & 1.0 & 0.33 & 0.02 & 0.32 & 0.15 & \textbf{0.35} & 0.06 \\
         Avila & 0.01 & 0.99 & 0.01 & 0.99 & 0.05 & 0.97 & \textbf{0.71} & 0.21 & 0.56 & 0.49 & 0.58 & 0.07 \\
         Bank marketing & 0.03 & 0.97 & 0.40 & 0.65 & \textbf{0.93} & 0.09 & 0.49 & 0.59 & 0.38 & 0.68 & 0.86 & 0.17 \\
         Adult & 0.02 & 0.98 & 0.11 & 0.93 & \textbf{0.95} & 0.08 & 0.68 & 0.44 & 0.15 & 0.89 & 0.89 & 0.15 \\
         Covtype & 0.01 & 0.99 & 0.03 & 0.97 & \textbf{0.68} & 0.36 & 0.13 & 0.89 & 0.09 & 0.92 & 0.41 & 0.03 \\
         \bottomrule
    \end{tabular}
\end{table*}

\begin{table*}[!htb]
    \centering
    \scriptsize
    \caption{LIME consistency scores and uniqueness ratio for various discretizations, averaged over 5 executions (std is lower than 0.08 in all cases).}
    \label{tab:lime_disc_full}
    \begin{tabular}{l | l l | l l | l l | l l | l l | l l }
        \toprule
         \multirow{2}{*}{Dataset} &
         \multicolumn{2}{c |}{Original} &
         \multicolumn{2}{c |}{2-FP} &
         \multicolumn{2}{c |}{1-FP} & 
         \multicolumn{2}{c |}{Sign} &
         \multicolumn{2}{c |}{Rank} &
         \multicolumn{2}{c}{Sign-of-top-5} \\
         & 
         Cons. & Uniq. &
         Cons. & Uniq. &
         Cons. & Uniq. &
         Cons. & Uniq. &
         Cons. & Uniq. &
         Cons. & Uniq. \\
         \midrule
         Heart & 0.0 & 1.0 & 0.0 & 1.0 & 0.34 & 0.81 & 0.02 & 0.99 &  0.0 & 1.0 & \textbf{0.46} & 0.66 \\
         Chess & 0.0 & 1.0 & 0.11 & 0.0 & 0.11 & 0.0 & \textbf{0.13} & 0.02 & \textbf{0.13} & 0.18 & \textbf{0.13} & 0.07 \\
         Avila & 0.0 & 1.0 & 0.23 & 0.0 & 0.23 & 0.0 & \textbf{0.37} & 0.17 & 0.01 & 0.98 & 0.28 & 0.33 \\
         Bank marketing & 0.0 & 1.0 & 0.0 & 1.0 & \textbf{0.87} & 0.0 & 0.65 & 0.48 & 0.0 & 1.0 & 0.84 & 0.10 \\
         Adult & 0.0 & 1.0 & 0.0 & 1.0 & \textbf{0.77} & 0.0 & \textbf{0.77} & 0.22 & 0.04 & 0.97 & 0.72 & 0.13  \\
         Covtype & 0.0 & 1.0 & 0.0 & 1.0 & 0.12 & 0.87 & 0.0 & 1.0 & 0.0 & 1.0 & \textbf{0.17} & 0.73  \\
         \bottomrule
    \end{tabular}
\end{table*}

\paragraph*{Counterfactuals}
Recall that for \emph{counterfatual} explanation, given an instance $x \in \mathbb{R}^d$ it returns a vector $x' \in \mathbb{R}^d$ such that $f(x) \neq f(x')$ and $x'$ is close to $x$. To obtain a counterfatual explanations we have used \texttt{DiCE}~\cite{mothilal2020dice}. As the space of explanations $\cE = \cX$ discretization of $\cE$ is essential for estimation of the explainability measures. To this end, we compared the following discretization methods.

\begin{itemize}
    \item \emph{Original}: return $x'$ as is.
    \item \emph{$\Delta$}: return $x' - x$, i.e. consider only the features that were modified.
    \item \emph{$\Delta$-sign}: return $x'' \in \mathbb{R}^d$, such that $x''_{i} = \mathtt{sign}(x'_{i} - x_{i})$.
    \item \emph{Is-feature-modified}: return $x'' \in \mathbb{R}^d$ such that $x''_i = \begin{cases} 1 & x_i = x'_i \\ 0 & else \end{cases}$.
\end{itemize}

Table~\ref{tab:counterfactuals_disc_full} depict the consistency and uniqueness ratio of the above discretizations.

\begin{table*}[!ht]
    \centering
    \scriptsize
    \caption{Counterfactuals consistency scores and uniqueness ratio for various discretizations, averaged over 5 executions (std is lower than 0.07 in all cases).}
    \label{tab:counterfactuals_disc_full}
    \begin{tabular}{l | l l | l l | l l | l l }
        \toprule
         \multirow{2}{*}{Dataset} &
         \multicolumn{2}{ c |}{Original} &
         \multicolumn{2}{ c |}{$\Delta$} &
         \multicolumn{2}{ c |}{$\Delta$-sign} &
         \multicolumn{2}{ c}{Is-feature-modified} \\
         & 
         Cons. & Uniq. &
         Cons. & Uniq. &
         Cons. & Uniq. &
         Cons. & Uniq. \\
         \midrule
         Heart & 0.0 & 1.0 & 0.01 & 1.0 & 0.19 & 0.88& \textbf{0.23} & 0.66 \\
         Chess & 0.09 & 0.70 & \textbf{0.14} & 0.17 & \textbf{0.14} & 0.02 & 0.13 & 0.01 \\
         Avila & 0.20 & 0.22 & 0.20 & 0.23 & \textbf{0.34} & 0.07 & 0.24 & 0.0 \\
         Bank marketing & 0.0 & 1.0 & 0.18 & 0.89 & \textbf{0.89} & 0.04 & 0.87 & 0.02 \\
         Adult & 0.0 & 1.0 & 0.07 & 0.96 & \textbf{0.91} & 0.04 & 0.81  & 0.01 \\
         Covtype & 0.0 & 1.0 & 0.38 & 0.46 & \textbf{0.57} & 0.03 & 0.52 & 0.02 \\
         \bottomrule
    \end{tabular}
\end{table*}

\end{document}